\documentclass[12pt]{article}
\usepackage{amsmath,amssymb,amsthm,booktabs}
\usepackage{mathrsfs}
\usepackage{geometry}
\usepackage{pifont}
\usepackage[OT1]{fontenc}
\usepackage[utf8]{inputenc}
\usepackage[colorlinks,citecolor=blue,urlcolor=blue]{hyperref}
\usepackage{txfonts}
\usepackage{indentfirst}
\usepackage{multirow}
\usepackage{bm}
\usepackage{euscript}
\usepackage{graphicx}
\usepackage{subcaption}
\usepackage{float}
\usepackage{multicol}
\usepackage[usenames,dvipsnames,svgnames,table]{xcolor}
\usepackage[round]{natbib}
\usepackage{algorithm}
\usepackage{algorithmic}
\numberwithin{equation}{section} \theoremstyle{plain}

\newtheorem{lemma}{Lemma}[section]
\newtheorem{corollary}{Corollary}[section]
\newtheorem{proposition}{Proposition}[section]

\newtheorem{remark}{Remark}[section]

\begin{document}

\newcommand{\gai}[1]{{#1}}

\makeatletter
\def\ps@pprintTitle{%
  \let\@oddhead\@empty
  \let\@evenhead\@empty
  \let\@oddfoot\@empty
  \let\@evenfoot\@oddfoot
}
\makeatother

\newcommand\tabfig[1]{\vskip5mm \centerline{\textsc{Insert #1 around here}}  \vskip5mm}
\vskip2cm

\title{Early-stopping for Transformer model training}
\author{\quad Jing He\thanks{Co-first Author, School of Mathematics, Shandong University, PR China, (hjing@mail.sdu.edu.cn).}
	\quad Hua Jiang\thanks{Co-first Author, School of Mathematics, Shandong University, PR China, (jh@mail.sdu.edu.cn).}
    \quad Cheng Li \thanks{Huawei Technologies Ltd. PR China, (cheng.li8@huawei.com). }
	\quad Siqian Xin\thanks{Shandong University-Zhong Tai Securities Institute for Financial Studies, Shandong University, PR China, (xinsiqian@mail.sdu.edu.cn).}
	\quad Shuzhen Yang\thanks{Corresponding author, Shandong University-Zhong Tai Securities Institute for Financial Studies, Shandong University, PR China, (yangsz@sdu.edu.cn).}
}
\date{}
\maketitle
\begin{abstract}
This work, based on Random Matrix Theory (RMT), introduces a novel early-stopping strategy for Transformer training dynamics. Utilizing the Power Law (PL) fit to tansformer attention matrices as a probe, we demarcate training into three stages: structural exploration, heavy-tailed structure stabilization, and convergence saturation. Empirically, we observe that the spectral density of the shallow self-attention matrix $V$ consistently evolves into a heavy-tailed distribution. Crucially, we propose two consistent and validation-set-free criteria: a quantitative metric for heavy-tailed dynamics and a novel spectral signature indicative of convergence. The strong alignment between these criteria highlights the utility of RMT for monitoring and diagnosing the progression of Transformer model training.
\end{abstract}

\newpage
\section{Introduction}
Behind the remarkable success of large language models (LLMs), exemplified by the Transformer architecture, lie substantial training costs and persistent risks of overfitting \citep{fournier2023practical,varis2021sequence}.
According to the well-established scaling law, the performance of an LLM is closely tied to its number of parameters, the quantity of training data, and the computational budget-factors that collectively drive the trend toward ever-larger models \citep{kaplan2020scaling}.
Within this context, early stopping serves as a crucial regularization technique for conserving computational resources and mitigating overfitting once the validation performance saturates\citep{kuken2025early,wu2025benefits,favero2025bigger}.

The classical early stopping approach relies on validation set performance \citep{prechelt2002early, muhammad2022early}.
However, this strategy exhibits clear limitations, particularly when the training and validation distributions diverge \citep{taori2020measuring} or when high-quality validation data are unavailable \citep{yuan2024early}.
\citet{heckel2021early} further revealed the double descent phenomenon, suggesting that naive validation monitoring may fail to capture the true optimal stopping point.
Consequently, recent research on early stopping methods has increasingly explored alternative approaches aimed at reducing or even eliminating the reliance on validation sets \citep{laaouach2025halt,miseta2024surpassing, ameryan2023limit}. One promising direction is to focus on intrinsic model signals that more accurately reflect a model’s generalization capability.
Among these, the spectral distribution of weight matrices-analyzed through the lens of Random Matrix Theory (RMT) has emerged as a powerful analytical tool.

Integrating RMT with regularization techniques enables model compression through spectral pruning without sacrificing accuracy, thereby demonstrating RMT’s potential to enhance deep learning performance \citep{berlyand2024enhancing, berlyand2025pruning}.
This principle of separating signal from noise provides a theoretical foundation for understanding and optimizing neural architectures.
For many state-of-the-art deep networks, the spectra of weight matrices exhibit heavy-tailed (HT) behavior, well-described by a power-law distribution \citep{meng2023impact, martin2021implicit}.
This phenomenon is not coincidental but reflects an implicit regularization mechanism that emerges during training—known as the Heavy-Tailed Self-Regularization (HT-SR) theory \citep{martin2021implicit}.
A key corollary of this theory is that ``shape metrics'' derived from spectral properties can directly predict a model’s generalization ability without access to training or test data.
Subsequent studies have confirmed the effectiveness of these metrics for model selection \citep{yang2023test} and linked them to complex learning behaviors such as ``Grokking'', underscoring the potential of spectral analysis \citep{prakash2025grokking}.

From a dynamic viewpoint, HT-SR theory characterizes deep neural network (DNN) training as an evolving process of spectral shape transformation.
\citet{martin2021implicit} proposed a ``5+1 phase'' framework to describe the progression of implicit self-regularization from weak to strong, while \citet{meng2023impact} later distilled this into a more practical three-stage model, revealing its close relationship to task difficulty.
This insight motivates a new type of early stopping criterion: DNN training can be terminated by monitoring changes in the spectral distribution shape.

Despite the Transformer’s widespread adoption \citep{vaswani2017attention}, most existing spectral studies have analyzed only the overall weight matrices of Transformer layers \citep{yang2023test, martin2021predicting}.
Few have investigated the internal spectral dynamics of the Transformer’s core component the attention mechanism specifically its Query (Q), Key (K), and Value (V) matrices.
Current Transformer-oriented early stopping methods either continue to rely on validation-based monitoring \citep{muhammad2022early} or focus on dynamic early-exit mechanisms for accelerating inference \citep{kuken2025early}.The Q, K, and V matrices play pivotal roles in information processing and representation learning within Transformer architectures \citep{anonymous2025from, borji2023key}. Their spectral dynamics are inherently linked to the model’s learning trajectory and ultimate generalization performance.
Building upon this observation, we propose a new research direction for early stopping: leveraging the spectral distributions and heavy-tailed characteristics of Q, K, and V matrices to infer the training state and derive a novel early stopping criterion.

We summarize our contributions as follows:
\begin{enumerate}
  \item Our analysis and empirical validation reveal that the self-attention V matrix from the first encoder layer (\texttt{en.0.s.a.V}) serves as a stable and sufficient proxy for monitoring the overall training dynamics of Transformer models.
Its spectral evolution exhibits both robustness and representativeness, obviating the need to track multiple matrices.
\item Based on the dynamics of \texttt{en.0.s.a.V} matrix, we conceptualize Transformer training as a novel three-phase process:
(i) Structural Exploration and Drastic Adjustment,
(ii) Heavy-Tailed Structure Formation and Stabilization, and
(iii) Performance Saturation and Convergence Plateau.
\item By fitting the spectral tail with a power-law (PL) distribution, we design a quantitative indicator that tracks the emergence of heavy-tailed structures and propose a spectral-based criterion for early stopping without requiring a validation set.
This criterion effectively identifies the optimal early stopping window during the late stabilization phase.

\end{enumerate}

The remainder of this paper is organized as follows.
Section \ref{sec:Formulation} introduces the formulation of the proposed spectral early stopping criterion.
Section \ref{sec:main section} describes the selection of the core matrix, the identification of three distinct training phases based on PL fitting, the derivation of the heavy-tailed recognition metric, and the final early stopping criterion.
Section \ref{sec:conclusion} concludes the paper, and the \nameref{Appendix} provides additional experimental details and theoretical analyses.

\section{Formulation}\label{sec:Formulation}
This paper aims to identify the training stage of Transformer models by determining the most representative core matrix (\texttt{en.0.s.a.V}) and analyzing its heavy-tailed characteristics. Based on this analysis, we establish a spectral-based early stopping criterion that guides the training process, deepens the understanding of the model’s internal dynamics, and effectively reduces training time.

To identify and analyze the most representative core matrix of the Transformer model, we begin with its fundamental components-the Query ($\mathbf{Q}$), Key ($\mathbf{K}$), and Value ($\mathbf{V}$) matrices.
Taking the Query matrix $\mathbf{Q}$ as an example (the same analysis applies to $\mathbf{K}$ and $\mathbf{V}$), we assume it has dimensions $N \times M$, where $N \geq M$; otherwise, we analyze its transpose.
The correlation matrix is defined as $\mathbf{X} = \mathbf{Q}^\top \mathbf{Q}$.
Let the eigenvalues of $\mathbf{X}$ be denoted by $\{\lambda_j\}_{j=1}^M$, where each $\lambda_j$ equals the square of the corresponding singular value $\sigma_j$ of $\mathbf{Q}$, i.e., $\lambda_j = \sigma_j^2$.
The Empirical Spectral Density (ESD) of $\mathbf{Q}$ is then defined as the empirical distribution of the eigenvalues of $\mathbf{X}$.
\begin{align}
\label{eq:ESD}
F^{Q}(x) \overset{\triangle}{=} \frac{1}{N}\sum_{i=1}^{N}\mathbf{I}(\lambda_i \le x),\quad x \in \mathbb{R},
\end{align}
where $\mathbf{I}(\cdot)$ is the indicator function.

Since the power law (PL) effectively characterizes the spectral tails of weight matrices exhibiting heavy-tailed behavior \citep{meng2023impact, martin2021implicit}, we perform PL fitting on all $\mathbf{Q}$, $\mathbf{K}$, and $\mathbf{V}$ matrices to evaluate the stability and consistency of their heavy-tailed evolution. The results indicate that the $\mathbf{V}$ matrix of the first self-attention layer (\texttt{en.0.s.a.V}) is the most representative, providing empirical validation for our hypothesis.

We further analyze the evolution of PL fitting parameters for the \texttt{en.0.s.a.V} matrix over 205 training epochs under three parameter settings. The spectral parameters exhibit highly consistent and distinct three-phase dynamics:
\begin{enumerate}
\item \textbf{Structural Exploration and Drastic Adjustment:} The model transitions from random initialization, during which the weight matrix undergoes large-scale restructuring.
\item \textbf{Heavy-Tailed Structure Formation and Stabilization:} The matrix develops a stable heavy-tailed spectral structure.
\item \textbf{Performance Saturation and Convergence Plateau:} The model’s primary feature-learning process is largely complete.
\end{enumerate}

Using the Kolmogorov-Smirnov (KS) statistic, we show that the convergence rate of the distance between the empirical cumulative distribution function (ECDF) of the \texttt{en.0.s.a.V} matrix and the theoretical cumulative distribution function (CDF) of the PL fit is $O_p \left( \frac{1}{\sqrt{n_{\text{tail}}}}\right)$. Consequently, when the KS distance $\tilde{d}$ between the ESD of \texttt{en.0.s.a.V} and the fitted power-law distribution falls below the threshold $d^* = \frac{C}{\sqrt{n_{\text{tail}}}}$, where $C$ is a critical constant estimated via Monte Carlo simulation, the ESD exhibits heavy-tailed behavior.
From this result, we derive the early stopping criterion
\begin{align*}
 \max_{\text{epoch}}{\{d^* - \tilde{d}\}},
\end{align*}
 where training should be terminated when $d^* - \tilde{d}$ reaches its maximum. At this point, the model achieves the strongest implicit self-regularization and optimal generalization capacity.

By integrating the Heavy-Tailed Self-Regularization (HT-SR) framework with spectral analysis of the Transformer’s core matrix \texttt{en.0.s.a.V}, we propose a robust early stopping criterion driven solely by intrinsic model signals. This approach provides deeper insight into the internal dynamics of large language models and facilitates improved model efficiency and generalization.
\section{Early Stopping Mechanism Based on PL Fitting}\label{sec:main section}
\subsection{Models setting and Datasets destails}
To investigate the evolution of the Transformer’s core matrices during training, we construct three Transformer variants with distinct architectural configurations, denoted as T1, T2, and T3 models. Except for these structural variations, all other hyperparameters remain identical across the models. The complete configuration details are provided in Table~\ref{tab:model_config}.
\begin{table}[htbp]
	\centering
    \small
	\caption{Model Architecture Configuration}
	\label{tab:model_config}
	\begin{tabular}{lccccc}
			\toprule
			\textbf{Model} & \textbf{Param} & \textbf{n\_layers} & \textbf{d\_model} & \textbf{d\_ff} & \textbf{n\_heads} \\
			\midrule
			T1 & 9.60M & 3 & 400 & 800  & 4 \\
			T2 & 9.62M & 2 & 400 & 1800 & 4 \\
			T3 & 19.20M & 6 & 400 & 800  & 4 \\
			\bottomrule
	\end{tabular}
\end{table}
To investigate the properties of the weight matrices in the Transformer model, we conduct experiments on the English-to-Chinese (En-Ch) machine translation task. The dataset statistics and relevant details are summarized in Table~\ref{tab:dataset_stats}. The translation corpus is obtained from a publicly available parallel dataset.\footnote{The English–Chinese dataset (CWMT) is available on Baidu AI Studio: \url{https://aistudio.baidu.com/datasetdetail/161162}}
In the subsequent experiments, the data are divided into training and validation sets with a 9:1 ratio.
\begin{table}[t]
	\centering
	\caption{Experimental Dataset Statistics}
	\label{tab:dataset_stats}
	\begin{tabular}{llcc}
		\toprule
		\textbf{Task} & \textbf{Language} & \textbf{Tokens (Millions)} & \textbf{Vocab Size} \\
		\midrule
		\multirow{2}{*}{En-Ch} & Chinese & 11.1 & 16,011 \\
		& English & 12.7 & 6,524 \\
		\bottomrule
	\end{tabular}
\end{table}
All experiments were conducted on the same high-performance computing platform.
To eliminate the influence of hardware variability, all models were trained using identical equipment.
During training, the models were optimized with the Adam optimizer, using a learning rate (\texttt{lr}) of 0.0001 and a batch size of 448.\footnote{Due to hardware constraints, we empirically found that a batch size of 448 minimizes training time while maintaining stable convergence.}
The maximum sequence lengths for both the encoder and decoder (\texttt{src\_len} and \texttt{tgt\_len}) were set to 30.
To further enhance generalization, a dropout rate of 0.2 was applied during training.

\subsection{Selection of Core Matrix by Power Law distribution for Transformer}\label{sec: choose and alanyis en.0.s.a.v}

For Transformer models, the attention matrices $\mathbf{Q}$, $\mathbf{K}$, and $\mathbf{V}$ are continuously optimized during training, making them central components of the model’s parameter space. However, synchronously tracking and analyzing all $\mathbf{Q}$, $\mathbf{K}$, and $\mathbf{V}$ matrices across every attention layer in large-scale models is computationally prohibitive. Moreover, not all matrices exhibit stable and consistent correlations with model performance throughout training. Therefore, it is necessary to identify the most representative matrix as an observation window to gain deeper insight into the internal dynamics of model training. The key question then arises: what criteria should guide the selection of such a matrix?

According to Random Matrix Theory (RMT), the spectral distribution of an untrained random matrix theoretically follows the Marchenko–Pastur (MP) distribution. However, our experiments reveal that even in the early stages of training, the weight matrices of Transformer models deviate substantially from the MP distribution. Instead, their empirical spectral densities (ESDs) exhibit prominent heavy-tailed behavior. This observation suggests that the MP distribution is not an adequate descriptor of weight-matrix properties in large language models \citep{prakash2025grokking}. To better capture this intrinsic heavy-tailed nature, we fit the ESDs of key matrices using a power-law distribution and use this fit as the criterion for identifying representative matrices. Specifically, this study evaluates the $\mathbf{Q}$, $\mathbf{K}$, and $\mathbf{V}$ matrices along two dimensions: \textbf{Significance of heavy-tailed characteristics}: Identify the matrix that exhibits the strongest heavy-tailed signature via spectral analysis, thereby justifying the use of the power-law model for quantitative characterization. \textbf{Stability of dynamic evolution}: Select the matrix whose power-law parameters evolve most smoothly and consistently during training. 

\citet{gurbuzbalaban2021heavy} established that the heavy-tailed phenomenon observed in the eigenvalue spectra of deep neural network weight matrices reflect an intrinsic regularization effect of gradient-descent-based optimization. \citet{song2024unraveling} further elucidated the distinct optimization dynamics among different weight matrices in Transformer models under quadratic loss. Building upon these theoretical foundations, we propose Proposition~\ref{prop:wv_convergence} (see \nameref{Appendix} \ref{sec: Choose V matrix}), demonstrating that analogous conclusions extend to machine translation tasks. Specifically, in a single-layer Transformer under over-parameterization (i.e., sufficiently large embedding dimension D) and proper initialization, when $\mathbf{Q}$ and $\mathbf{K}$ remain fixed while only $\mathbf{V}$ undergoes gradient descent updates, the cross-entropy loss converges linearly to the global optimum. This behavior originates from the softmax kernel in the attention mechanism: while the optimization trajectories of $\mathbf{Q}$ and $\mathbf{K}$ are nonlinearly coupled with the kernel, the dynamics of $\mathbf{V}$ remain approximately linear, thus providing a cleaner observational window.

Guided by this theoretical insight, we advance the hypothesis that the spectral distribution of $\mathbf{V}$ more reliably reflects the model’s implicit regularization state. To empirically validate this hypothesis and determine which attention matrix exhibits the most pronounced heavy-tailed behavior, we design a systematic comparative analysis framework.

For a given attention matrix \(\mathbf{W}\), the fitting procedure is implemented as follows:

\textbf{Step 1: Eigenvalue Selection}
\par All eigenvalues of the correlation matrix \(\mathbf{X} = \mathbf{W}^\top\mathbf{W}\) are computed. The Kolmogorov-Smirnov (KS) test is then applied to determine the optimal lower bound for model fitting, denoted as \(x_{\min}\). Subsequent analysis focuses on the eigenvalue tail, comprising all eigenvalues satisfying  \(\geq x_{\min}\).

\textbf{Step 2: Model Fitting}
\par Two candidate distributions are fitted to this eigenvalue tail. Using maximum likelihood estimation (MLE), we estimate the optimal exponent \(\alpha\) for the power-law model and the decay rate \(\lambda\) for the exponential model.

Two candidate distributions are fitted to this eigenvalue tail. Using maximum likelihood estimation (MLE), we estimate the optimal exponent $\alpha$ for the power-law model and the decay rate $\lambda$ for the exponential model, respectively.

For the continuous power-law probability density function,

\begin{equation}\label{eq:PL PDF}
    p(x; \alpha, x_{\min}) = \frac{\alpha-1}{x_{\min}}\left(\frac{x}{x_{\min}}\right)^{-\alpha},
\end{equation}
where \(\alpha\) represents the scaling parameter (or Power-Law exponent), and \(x_{\min}\) denotes the lower threshold above which power-law behavior emerges. The MLE for \(\alpha\) is given by:
\begin{equation}
    \hat{\alpha} = 1 + n \left[ \sum_{i=1}^{n}\ln\frac{x_i}{x_{\min}} \right]^{-1}.
\end{equation}

For the continuous exponential probability density function,
\begin{equation}
	p(x; \lambda, x_{\min}) = \lambda e^{-\lambda (x - x_{\min})}, \label{eq:exp_pdf}
\end{equation}
where $\lambda$ represents the decay rate, and $x_{\min}$ denotes the lower threshold as defined above. The MLE for $\lambda$ is given by:
\begin{equation}
	\hat{\lambda} = n \left[ \sum_{i=1}^{n} (x_i - x_{\min}) \right]^{-1}. \label{eq:exp_mle}
\end{equation}

Following \citet{alstott2014powerlaw}, we posit that if a power-law model does not significantly outperform an exponential model in fitting empirical data, there is insufficient evidence to classify the distribution as heavy-tailed. Consequently, we employ the exponential distribution as a competitive baseline and use log-likelihood ratio testing to determine which model more accurately captures heavy-tailed phenomena in the spectral data.If \(\mathcal{R} > 0\), the data are more likely to follow a power-law distribution than an exponential one, providing evidence of heavy-tailed behavior. Otherwise, heavy-tailedness cannot be inferred.To assess the statistical significance of this result, we adopt the procedure proposed by \citet{clauset2009power} to compute a  \(p\)-value. If \(p < 0.1\) (a commonly used significance threshold), the empirical data are considered significantly more consistent with a power-law distribution, indicating that the matrix exhibits statistically significant heavy-tailed spectral characteristics.  

We summarize the procedure in Algorithm \ref{alg:test_power_law_superiority}. Due to space constraints, detailed fitting results are deferred to \nameref{Appendix}~\ref{sec:choose en.0.s.a.v}. Crucially, these results reveal that the \texttt{en.0.s.a.v} matrix consistently exhibits pronounced heavy-tailed characteristics throughout the training process. This empirical evidence fulfills our objective of identifying the matrix with \textbf{the strongest heavy-tailed signature}.

\begin{algorithm}[t]
	\caption{Test for Power-Law Superiority}
	\label{alg:test_power_law_superiority}
	\begin{algorithmic}
		\STATE {\bfseries Input:} Weight matrix $\mathbf{W}$, Significance threshold $p_{\text{thres}} = 0.1$
		\STATE {\bfseries Output:} $\text{is\_heavy\_tailed}$, $x_{\min}$, $\hat{\alpha}$, $\hat{\lambda}$, $\mathcal{R}$, $p\text{-value}$
		
		\STATE \textbf{Step 1: Eigenvalue Selection}
		\STATE \quad $\mathbf{X} = \mathbf{W}^\top \mathbf{W}$
		\STATE \quad $\text{Eigval} = \text{compute\_eigval}(\mathbf{X})$
		\STATE \quad $x_{\min} = \text{ks\_test}(\text{Eigval}, \text{target=``power\_law''})$
		\STATE \quad $\text{Eigval\_tail} = \{\xi \in \text{Eigval} \mid \xi \geq x_{\min}\}$
		
		\STATE \textbf{Step 2: Model Fitting}
		\STATE \quad $\hat{\alpha} = \text{mle\_pl}(\text{Eigval\_tail}, x_{\min})$
		\STATE \quad $\hat{\lambda} = \text{mle\_exp}(\text{Eigval\_tail})$
		
		\STATE \textbf{Step 3: Model Comparison and Heavy-Tailedness Test}
		\STATE \quad $\log\mathcal{L}_{\text{PL}} = \sum_{\xi \in \text{Eigval\_tail}} \ln\left(\text{PL}_{\text{PDF}}(\xi, \hat{\alpha}, x_{\min})\right)$
		\STATE \quad $\ln\mathcal{L}_{\text{Exp}} = \sum_{\xi \in \text{Eigval\_tail}} \ln\left(\text{EXP}_{\text{PDF}}(\xi, \hat{\lambda})\right)$
		\STATE \quad $\mathcal{R} = \ln\mathcal{L}_{\text{PL}} - \ln\mathcal{L}_{\text{Exp}}$
		\STATE \quad $p\text{-value} = \text{clauset2009\_p}(\text{Eigval\_tail}, \hat{\alpha}, x_{\min}, \hat{\lambda})$
		\STATE \quad $\text{is\_heavy\_tailed} = (\mathcal{R} > 0) \land (p\text{-value} < p_{\text{thres}})$
		
		\STATE \textbf{Return} $\text{is\_heavy\_tailed}, x_{\min}, \hat{\alpha}, \hat{\lambda}, \mathcal{R}, p\text{-value}$
	\end{algorithmic}
\end{algorithm}

Furthermore, we performed power-law parameter estimation on the \(\mathbf{Q}\), \(\mathbf{K}\), and \(\mathbf{V}\) matrices across all attention layers of the Transformer models (T1\_En-Ch9.60M, T2\_En-Ch9.62M). The corresponding results regarding the \textbf{stability of dynamic evolution} are detailed in \nameref{Appendix}~\ref{sec:en.0.s.a.v more stable}.

Comprehensive experimental analysis reveals a consistent pattern: the \textbf{\(\mathbf{V}\) matrix in the first self-attention layer of the encoder} (\texttt{en.0.s.a.v}) exhibits the most pronounced and stable heavy-tailed behavior, with clearly evolving spectral parameters and strong cross-model consistency.

\subsection{Three Training Phases for Transformer model training}\label{sec:Three training phases}
As discussed in Section~\ref{sec: choose and alanyis en.0.s.a.v}, the \texttt{en.0.s.a.v} matrix exhibits a more pronounced heavy-tailed phenomenon, and its power-law (PL) fitting demonstrates a highly stable evolutionary pattern. This predictable evolution makes the \texttt{en.0.s.a.v} matrix an ideal observation window for uncovering the intrinsic relationship between the dynamics of this layer and the overall convergence state of the model. Based on this finding, we select the \texttt{en.0.s.a.v} matrix as the primary object of analysis to further investigate the evolution of its spectral parameters throughout training.

The analysis in \nameref{Appendix} \ref{sec: heavy-tail}  has shown that the fitted parameters of the model begin to stabilize after approximately 50 epochs. To capture this transition precisely—from an initial fluctuating phase to a convergent state—we analyze 1-55,100-105,150-155,200-205 training epochs of this layer, focusing on the evolution of its PL fitting parameters. 
The detailed trajectories of the fitted parameters ($\alpha$ and $x_{min}$) for the \texttt{en.0.s.a.v} matrix, tracked under the experimental configuration \textbf{T3\_En-Ch19.20M} as a function of training epochs, are shown in Figure~\ref{fig:alpha xmin for en.0.s.a.v}; the corresponding plots for \textbf{T1\_En-Ch9.60M} and \textbf{T2\_En-Ch9.62M}\footnote{\textbf{T1\_En-Ch9.60M:} T1 model (En-Ch dataset). \textbf{T2\_En-Ch9.62M:} T2 model (En-Ch dataset). \textbf{T3\_En-Ch19.20M:} T3 model (En-Ch dataset).}
 are provided in Figure~\ref{fig:alpha xmin for en.0.s.a.v in appendix} from \nameref{Appendix} \ref{supplementary_three_phases}.

\begin{figure}[h]
	\centering \includegraphics[width=1\linewidth]{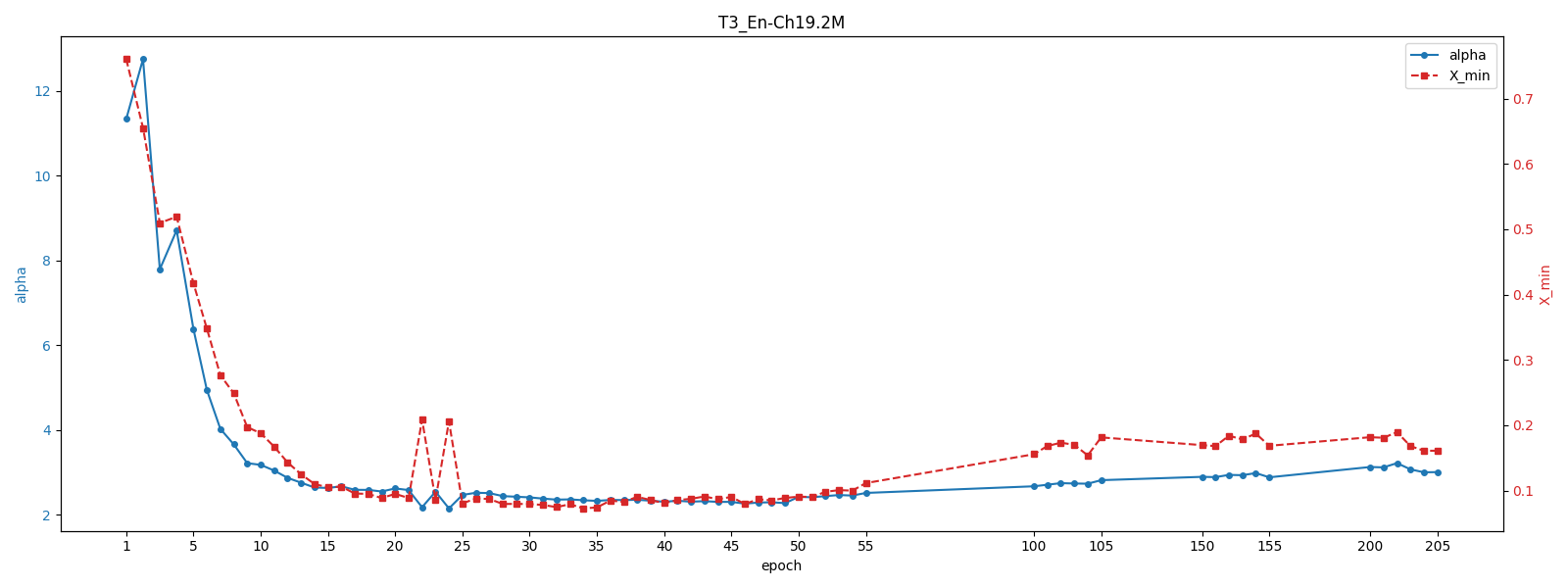}
	\label{fig:alpha xmin for en.0.s.a.v in T1_En-Ch 9.62M}
	\caption{Evolution of $\alpha$ and $x_{\text{min}}$ for \texttt{en.0.s.a.v}  matrix of \textbf{T2\_En-Ch9.62M}. The blue solid line (left Y-axis) represents the parameter $\alpha$, while the red dashed line (right Y-axis) corresponds to the PL lower bound parameter $x_{\text{min}}$.}
	\label{fig:alpha xmin for en.0.s.a.v}
\end{figure}

The power-law (PL) fitting parameter $\alpha$ is a key metric for evaluating model training stability and feature learning quality \citet{martin2021predicting}. To investigate the consistency of model training evolution under different parameter settings, we conducted a density statistical analysis on the $\alpha$ values derived from PL fitting of the \texttt{en.0.s.a.v} matrices across three groups of experiments (T1, T2, T3), with the corresponding density histograms shown in Figure \ref{fig:nofix}. The results indicate that the $\alpha$ values of the three experiments are mainly concentrated in the range of 2–4. Among them, the density histograms of Experiments T1 and T2 (with identical parameter settings) exhibit a high degree of overall similarity, demonstrating that under consistent parameter configurations, the model's training evolution process possesses extremely high reproducibility and similarity. In contrast, during the training of Model T3 (with different parameter settings), obvious outliers of $\alpha$ values in the range of 10–13 were observed, these outliers correspond to the initial stage of model training.The T3 model with a larger parameter scale requires significant global structural reorganization to complete the initial construction of an effective feature space, a process that directly leads to the abnormal deviation of $\alpha$ values.

The optimal fitting threshold \(x_{min}\) in the power-law (PL) fitting process varies with the dynamic update of the \texttt{en.0.s.a.v} matrix, while in the later stages of model training, \(x_{min}\) exhibits reduced fluctuations and tends to stabilize. To eliminate the interference of $x_{min}$ fluctuations on the statistical results of $\alpha$ and ensure the consistency of the analysis, we fixed $x_{min} = 0.1$ and re-performed PL fitting and density statistics. At this point, the PL exponents ($\alpha$) of the \texttt{en.0.s.a.v} matrices for the three models are mainly concentrated between 1.8 and 2.6, showing a more centralized distribution, with the corresponding density distribution diagrams presented in Figure \ref{fig:fix}.
\begin{figure}[htbp]
	\centering
	\begin{subfigure}[b]{0.45\linewidth}
		\centering
  \includegraphics[width=1\linewidth]{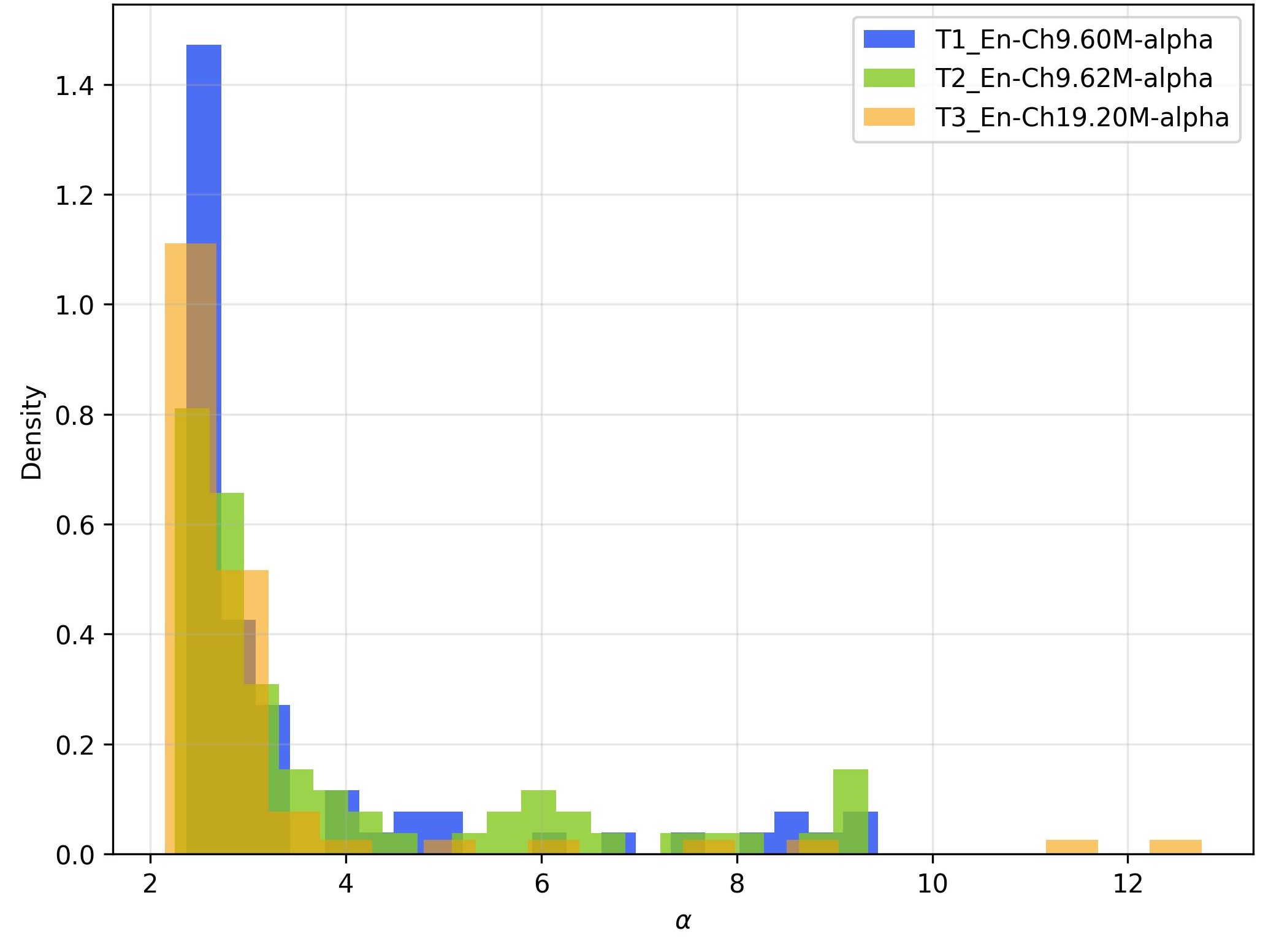}
  \caption{No fix $x_{min}$}\label{fig:nofix}
	\end{subfigure}
	\begin{subfigure}[b]{0.45\linewidth}
		\centering
  \includegraphics[width=0.97\linewidth]{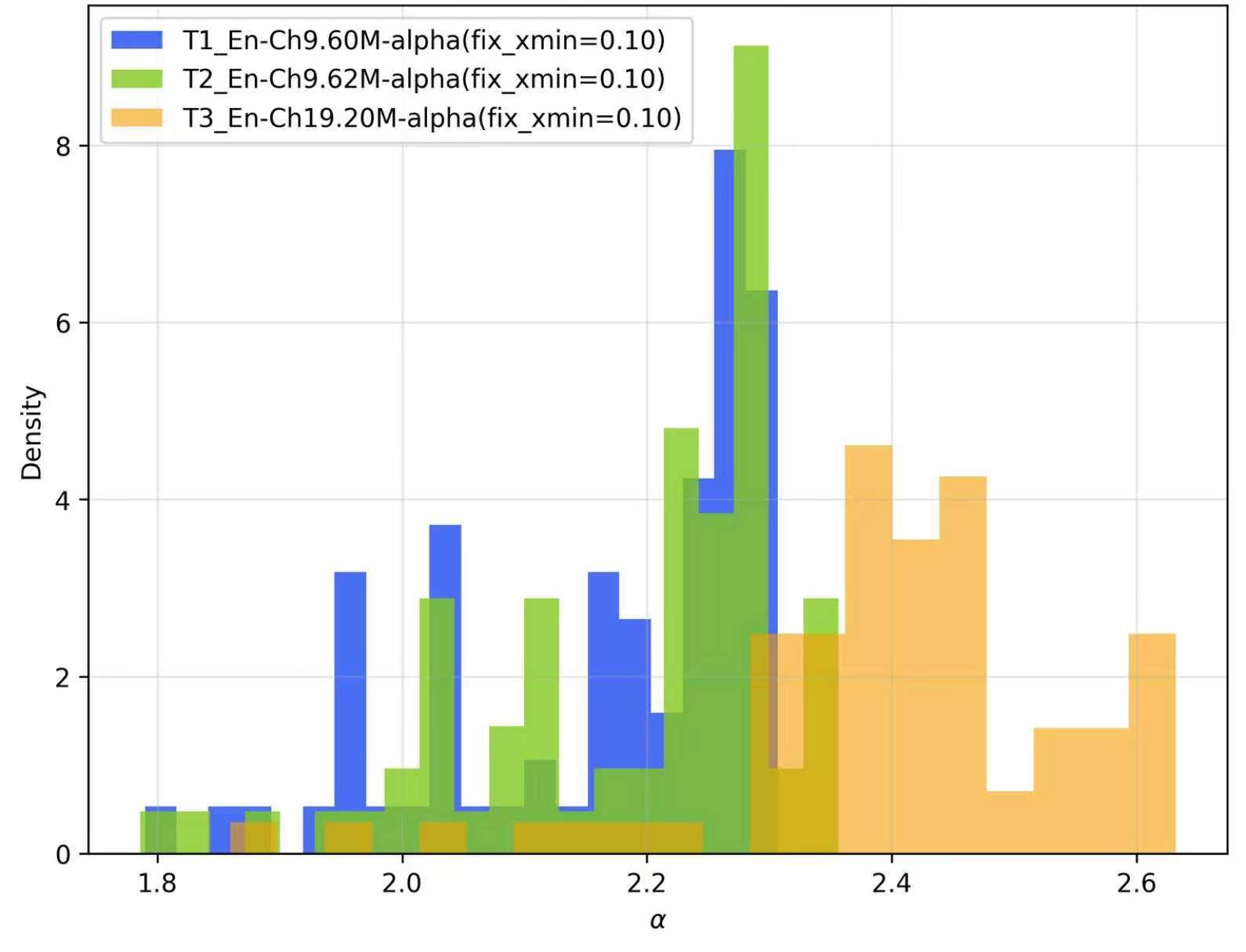}
  \caption{Fix $x_{min} = 0.1$}\label{fig:fix}
	\end{subfigure}
    \caption{PL epxonents density histograms of T1, T2, and T3 models}
\end{figure}

From the statistical results after fixing $x_{min} = 0.1$ (selected based on the approximate empirical values from the stable heavy-tailed phase of the three experimental groups), the $\alpha$ distributions of Models T1 and T2 (with the same parameter settings) highly overlap, indicating that the feature extraction logic and effects of models with identical parameters have good consistency. In contrast, the $\alpha$ values of Model T3 are generally larger, showing a significant right-shift feature in its distribution, but the distribution patterns of the three models still maintain a certain degree of similarity. The above results confirm that although there are slight differences in feature extraction accuracy and the extent of generalization ability improvement among models with different parameters, the evolutionary laws of the core feature extraction process have inherent consistency. This empirical result provides critical support for the rationality and generality of the following three stages of model training and the unified early stopping guidance framework proposed in this paper.

The spectral parameters of the \texttt{en.0.s.a.v}  matrix exhibit a highly consistent and distinct three-phase evolutionary dynamic across all three experimental setups.

\textbf{Phase I: Structural Exploration and Drastic Adjustment.}

During the initial training epochs, the spectral parameters undergo the most substantial fluctuations. Both the PL exponent $\alpha$ and the minimum eigenvalue threshold $x_{\min}$ remain elevated and exhibit considerable volatility. This behavior indicates that the model is transitioning away from its random initialization, undergoing large-scale global structural reorganization to rapidly explore and locate an effective feature space.

\textbf{Phase II: Heavy-Tailed Structure Formation and Stabilization.}
Following the initial exploration, the system enters a prolonged phase of stable evolution. The hallmark of this phase is the convergence of $\alpha$ to a low, stable value (consistently $\alpha < 3$, clustering around $2.5$), which is maintained over hundreds of epochs. This stabilization signals the formation of a robust heavy-tailed spectral structure within the weight matrices—a key mechanism linked to \textit{Implicit Regularization} in deep learning. Concurrently, $x_{\min}$ decreases and stabilizes at a low magnitude (approximately $0.1$).

\textbf{Phase III: Performance Saturation and Convergence Plateau.}
In the later training stages, the model enters a saturation phase. The most notable parametric signature is a gradual upward trend in $x_{\min}$, which correlates with the diminishing returns in performance gains. This suggests that the primary feature-learning process is largely complete; subsequent training serves primarily to fine-tune the model, indicating a convergence plateau.

Based on the evolutionary patterns illustrated in Figure~\ref{fig:alpha xmin for en.0.s.a.v}, the training processes for the three experimental configurations can be quantitatively segmented into these distinct phases. The corresponding phase durations and characteristics are summarized in Table~\ref{tab:4}

\begin{table}[htbp]
	\centering
	\caption{Epoch-Based Division of Training Stages for Each Model}
	\label{tab:4}
	\begin{tabular}{c|ccc}
		\hline \begin{tabular}[c]{@{}c@{}}\textbf{Model}\end{tabular} &
		\begin{tabular}[c]{@{}c@{}}\textbf{Phase I} \end{tabular} &
		\begin{tabular}[c]{@{}c@{}}\textbf{Phase II} \end{tabular} &
		\begin{tabular}[c]{@{}c@{}}\textbf{Phase III} \end{tabular} \\ \hline
		T1 &  1--25 &  26--105 &  150--205 \\
		T2 &  1--27 &  28--155 &  200--205 \\
		T3 &  1--24 &  25--55 &  100--205 \\ \hline
	\end{tabular}
\end{table}

To visually examine the distinct spectral distribution characteristics across the three phases, we use the T3\_En-Ch19.20M experiment as a representative case, selecting the 1st, 50th, and 200th epochs as representative time points. Figure~\ref{fig:esd_phases_mp} presents the Empirical Spectral Density (ESD) of the \texttt{en.0.s.a.v} matrix along with its corresponding Power-Law (PL) fit at these epochs, including a magnified view of the tail to clearly illustrate the typical spectral morphology at different training stages.
\begin{figure}[htbp]
	\centering
	\begin{subfigure}[b]{0.32\linewidth}
		\centering \includegraphics[width=1\linewidth]{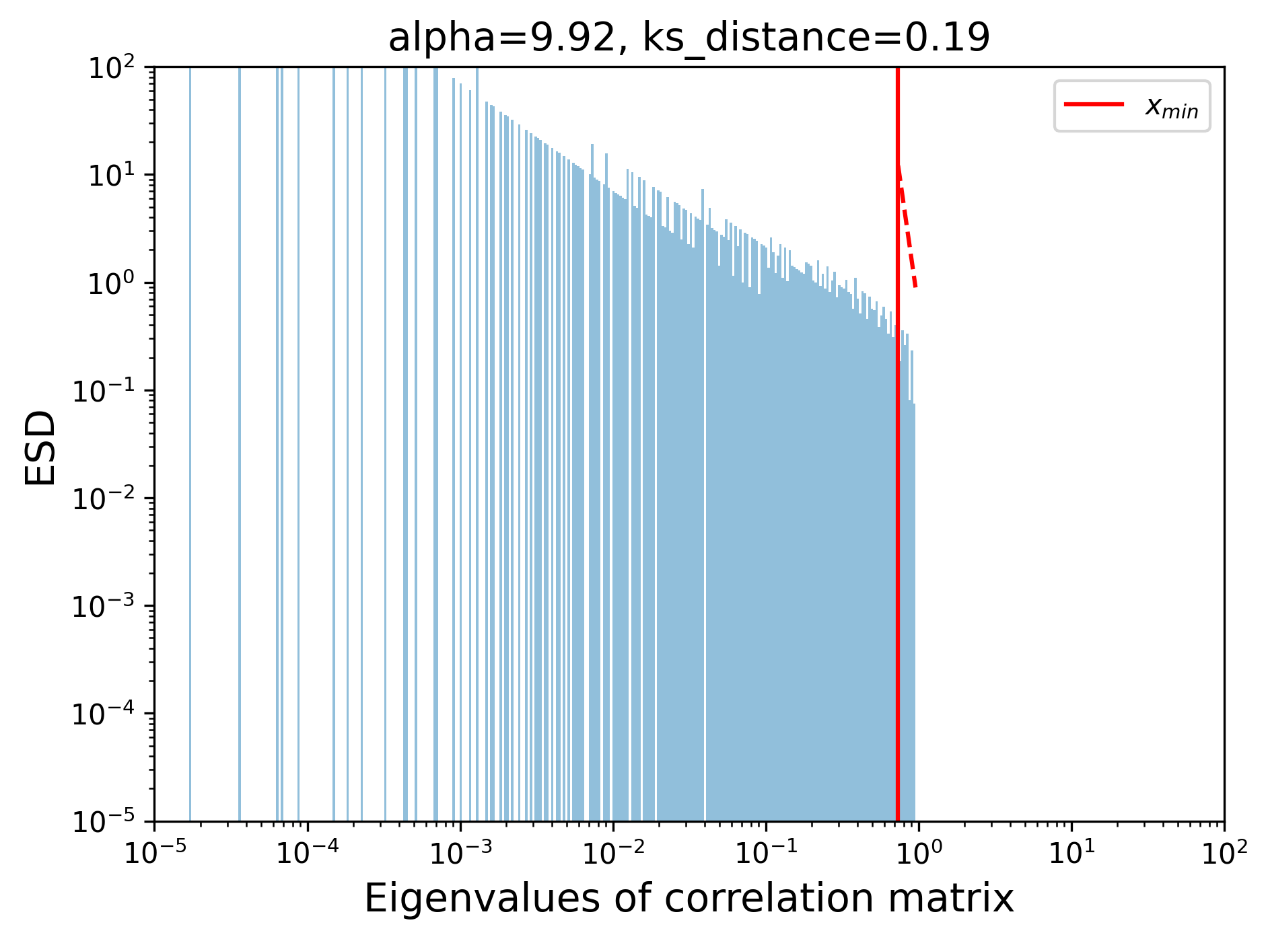}
		\caption{}
		\label{fig:esd1}
	\end{subfigure}
	\begin{subfigure}[b]{0.32\linewidth}
		\centering \includegraphics[width=1\linewidth]{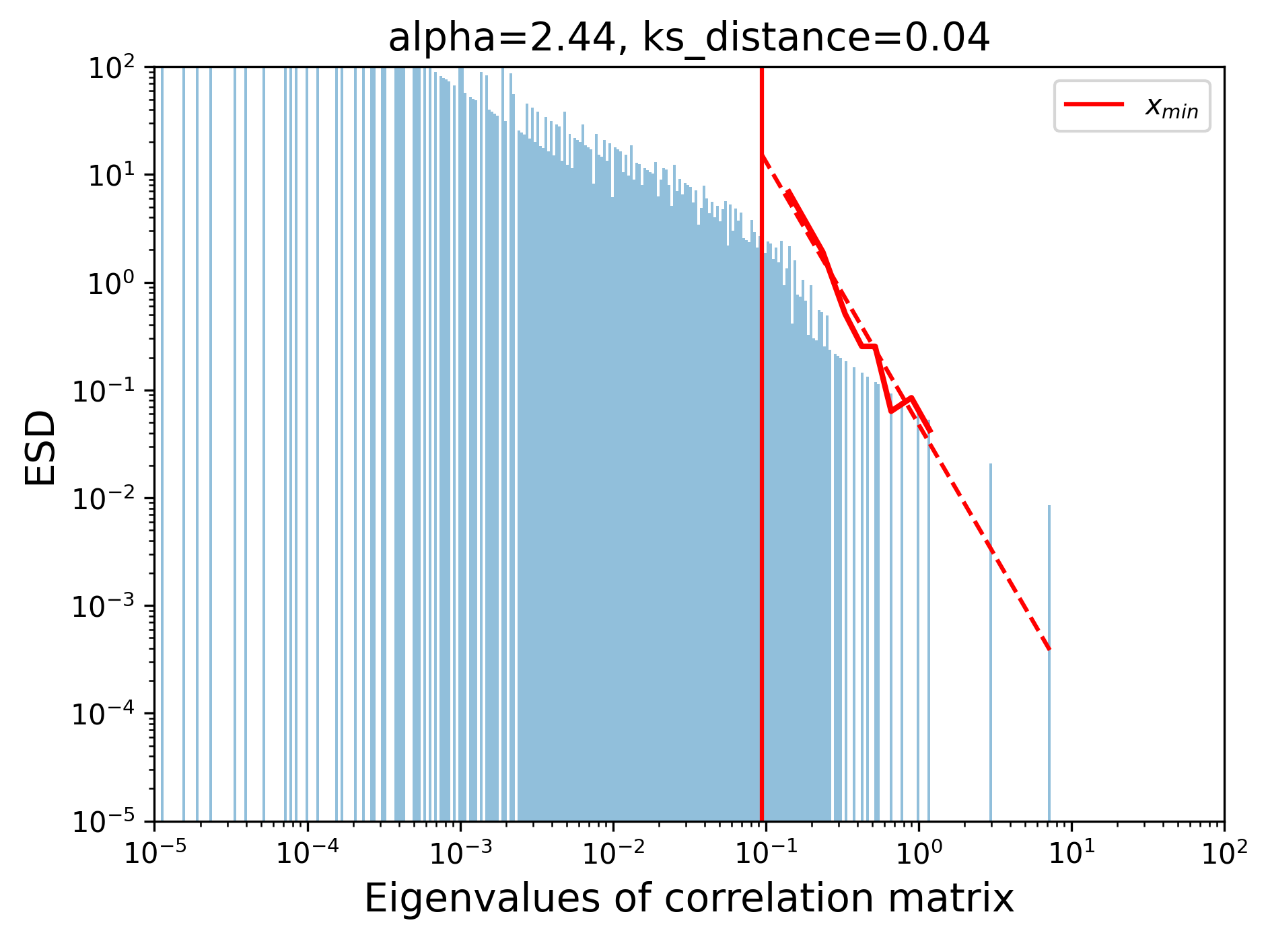}
		\caption{}
		\label{fig:esd2}
	\end{subfigure}
	\begin{subfigure}[b]{0.32\linewidth}
		\centering \includegraphics[width=1\linewidth]{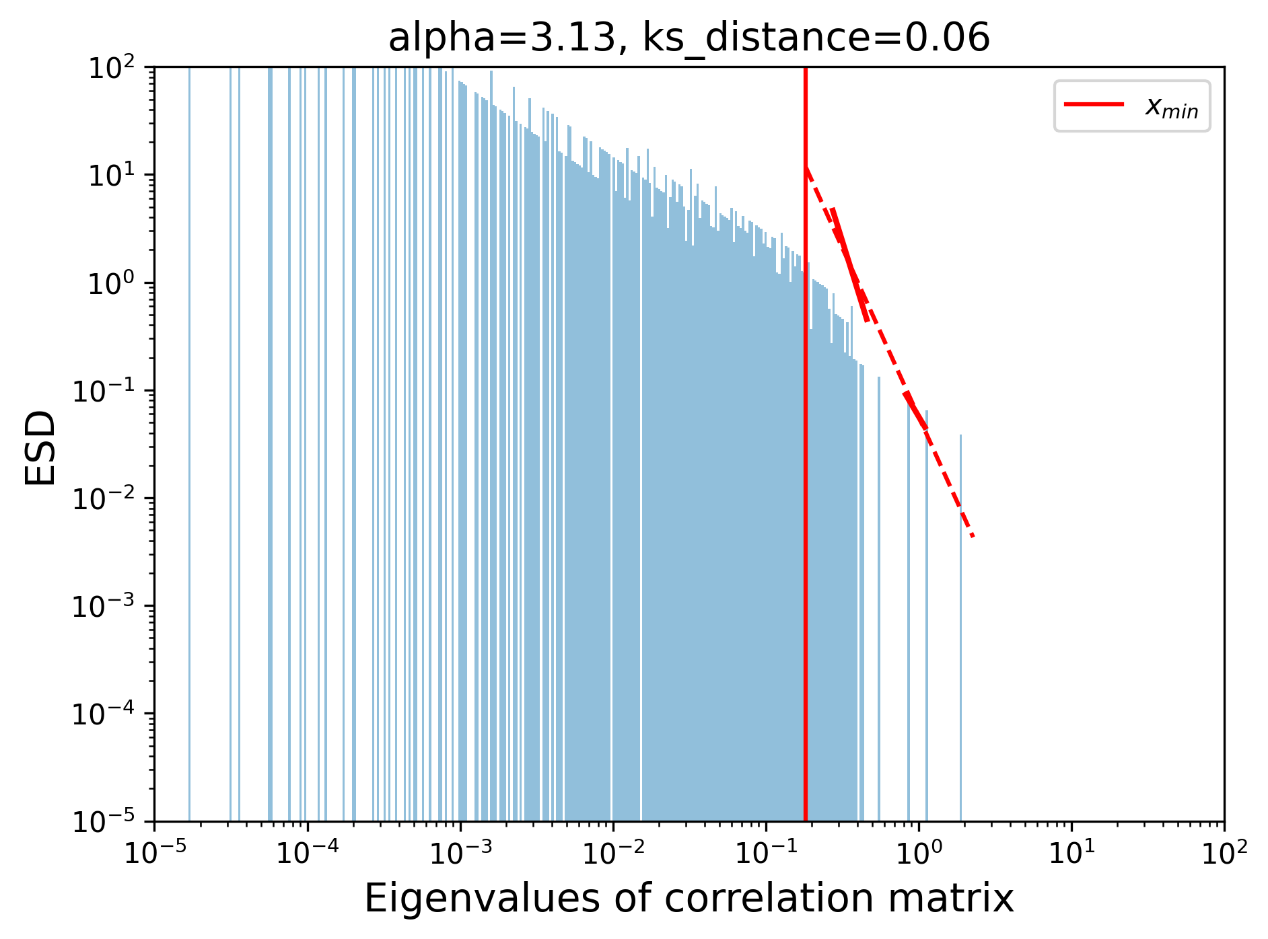}
		\caption{}
		\label{fig:esd3}
	\end{subfigure}
	\caption{ESD and PL fits for the \texttt{en.0.s.a.v} matrix from the T3\_En-Ch19.20M experiment at three representative training phases(1 st, 50th and 200th epochs).}
	\label{fig:esd_phases_mp}
\end{figure}

In Phase I (1st epoch, Fig.~\ref{fig:esd1}), the spectral tail is steep and short. The PL fit yields a very high $\alpha$ value of 9.92, indicating that the weight matrix has not yet developed heavy-tailed characteristics.

Upon entering Phase II (50th epoch, Fig.~\ref{fig:esd2}), the spectrum exhibits pronounced heavy-tailed behavior, reflecting a strong HT regime. The tail extends smoothly toward larger eigenvalues and displays high dispersion. The fitted $\alpha$ value drops sharply to 2.44, marking the formation of a strong heavy-tailed distribution.

By Phase III (200th epoch, Fig.~\ref{fig:esd3}), the heavy-tailed structure persists but is attenuated relative to Phase II, representing a weaker HT regime. Notably, the starting point of the PL fit, 
$x_{\min}$, shifts rightward, while the $\alpha$ value rebounds to 3.13. This contraction of the spectral tail, with fewer and smaller extreme eigenvalues, coincides with the model entering its performance saturation phase.

In \nameref{Appendix} \ref{supplementary_three_phases}, we compare the three divided training phases with the variations of loss values and accuracy during the model training process, thereby further verifying the effectiveness of our phase division criteria.

Through the above analysis, we can conclude that, for Transformer models, the \textbf{optimal early-stopping window} lies between the late portion of Phase II and the early portion of Phase III—after the heavy-tailed structure has stabilized but before the model reaches full performance saturation.
\subsection{Construction of a Power-Law-Based Heavy-Tailed Metric: Key Theoretical Details}

This section formulates a quantitative spectral metric based on the goodness-of-fit of the Power-Law distribution, establishing a statistical threshold to rigorously distinguish heavy-tailed characteristics.

We focus on tail characteristics of \texttt{en.0.s.a.v}, specifically, the portion where the eigenvalues \(x \geq x_{\min}\). All \(n_{\text{tail}}\) data points satisfying this condition can be regarded as being independently and identically drawn (i.i.d.) from an ideal PL distribution.
Based on Eq \eqref{eq:PL PDF}, the Cumulative Distribution Function (CDF) of this idealized power-law distribution is defined as follows,
\[
F(x; \alpha, x_{\min}) = \int_{x_{\min}}^{x} p(t) dt = 1 - \left(\frac{x}{x_{\min}}\right)^{-\alpha+1}.
\]
The Empirical Cumulative Distribution Function (ECDF) is
\[
\hat{F}_{n_{\text{tail}}}(x) = \frac{1}{n_{\text{tail}}}\sum_{i=1}^{n_{\text{tail}}} I(x_i \leq x),
\]
where \(I(\cdot)\) is the indicator function. To quantify the agreement between the empirical eigenvalues distribution and the theoretical distribution, we employ the Kolmogorov-Smirnov (KS) statistic. The KS statistic \(d\) is defined as the maximum vertical distance between the ECDF and the theoretical CDF over the interval \(x \geq x_{\min}\),
\begin{equation}\label{eq:KS formula}
	d = \sup_{x \geq x_{\min}} |\hat{F}_{n_{\text{tail}}}(x) - F(x; \alpha, x_{\min})|.
\end{equation}
Under the null hypothesis that the data points are drawn from a power-law distribution, \( d \)  converges to zero at a well-defined rate.

\begin{lemma}
	\label{lem:covergence of d}
	Suppose \( \{X_i\}_{i=1}^n \) are generated independently from \( p(x; \alpha, x_{\min}) = \frac{\alpha - 1}{x_{\min}}\left(\frac{x}{x_{\min}}\right)^{-\alpha} \), then the distance in Eq \eqref{eq:KS formula} satisfies
	\begin{equation}
		d = \sup_{x \geq x_{\min}} |\hat{F}_{n_{\text{tail}}}(x) - F(x; \alpha, x_{\min})| = O_p\left(\frac{1}{\sqrt{n_{\text{tail}}}} \right).
	\end{equation}
	(Here \( O_p \) denotes the boundedness in probability.)
\end{lemma}
The proof of the Lemma is given in Appendix \ref{pro:lemma}.
In practical applications, the PL exponent \(\alpha\) is typically unknown and must be estimated from sample data, denoted as \(\hat{\alpha}\). The CDF yields the empirical KS distance is
\begin{equation}\label{eq:MLE of KS formula}
	\hat{d} = \sup_{x \geq x_{\min}} |\hat{F}_{n_{\text{tail}}}(x) - F(x; \hat{\alpha}, x_{\min})|.
\end{equation}

The following proposition guarantees a convergence rate for the estimator \( \hat{d} \),
\begin{proposition}
	\label{prop:5.2}
	For the estimated distance \( \hat{d} \) in Eq \eqref{eq:MLE of KS formula}, we have
	\begin{equation}
		\hat{d} = O_p\left(\frac{1}{\sqrt{n_{\text{tail}}}}\right).
	\end{equation}
\end{proposition}

The proof of the proposition is given in Appendix \ref{pro:prop}.
By Proposition~\ref{prop:5.2}, under the null hypothesis that the ESD of the \texttt{en.0.s.a.v} matrix follows a Power-Law distribution (i.e., exhibits Heavy-Tailed characteristics), the KS distance $\hat{d}$ between the empirical spectral distribution and the fitted Power-Law distribution converges to 0. By defining the corresponding rejection region for the null hypothesis, we can conclude that if the ESD of the weight matrix deviates from the Power-Law form - i.e., in the absence of Heavy-Tailed characteristics -
$\hat{d}$ will no longer tend to 0. This observation leads directly to the following corollary.
\begin{corollary}\label{corollary: crition of HT}
Empirical distance $\tilde{d}$ is the KS distance between the ESD of \texttt{en.0.s.a.v} and the fitted Power-Law distribution, i.e.,
\begin{align*}
	\tilde{d} = \sup_{x \geq x_{\min}} |\hat{F}_{n_{\text{tail}}}(x) - F(x; \hat{\alpha}, x_{\min})|.
\end{align*}
If $\tilde{d} \le d^*$,  the ESD of \texttt{en.0.s.a.v} matrix exhibits a Heavy-Tailed property; otherwise, it does not, where threshold distance $d^*$ is
  \begin{equation}
	d^* = \frac{C}{\sqrt{n_{\text{tail}}}}.
 \end{equation}
 (C is a critical constant)
\end{corollary}
We determined the value of C = 2 via Monte Carlo simulation, as detailed in the \nameref{Appendix}~\ref{sec:Monte Carlo simulation}.

\subsection{Early Stopping criterion and its Efficacy}
In this section, we aim to construct a novel spectral criterion to identify the optimal early-stopping epoch for guiding Transformer model training. We find that the PL fitting-based early-stopping criterion aligns well with the phase-partitioning-based early-stopping window described in Section~\ref{sec:Three training phases}.

According to Corollary~\ref{corollary: crition of HT}, with the tail threshold fixed at $x_{\min} = 0.1$ for our experiments, the Heavy-Tailed (HT) indicator, defined as the difference $ d^* - \tilde{d}$ between the threshold distance and the empirical KS distance, can effectively monitor the dynamic evolution of model training and track the formation of Heavy-Tailed spectral structures. The transition of this indicator from negative to positive values signals the entry of the spectral distribution into the heavy-tailed regime.

We propose an \textbf{early-stopping criterion} based on $\mathbf{\max_{epoch}\{d^* - \tilde{d}\}}$, where the HT indicator \( d^* - d \) reaches its maximum value. This epoch is considered the optimal early-stopping point, at which the model achieves the highest degree of implicit self-regularization and the strongest generalization capability. At this point, the model’s objectives of performance optimization and training time minimization are optimally balanced.
\begin{figure}[htbp]
	\centering
	\begin{subfigure}[b]{0.49\linewidth}
		\centering
		\includegraphics[width=1\linewidth]{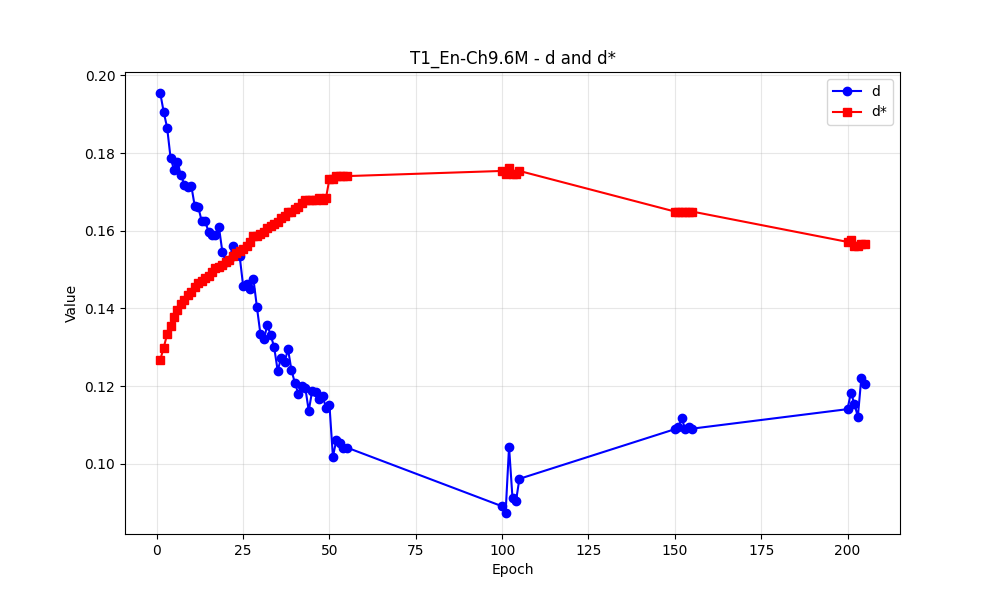}
		\caption{}
	\end{subfigure}
	\begin{subfigure}[b]{0.49\linewidth}
		\centering
		\includegraphics[width=1\linewidth]{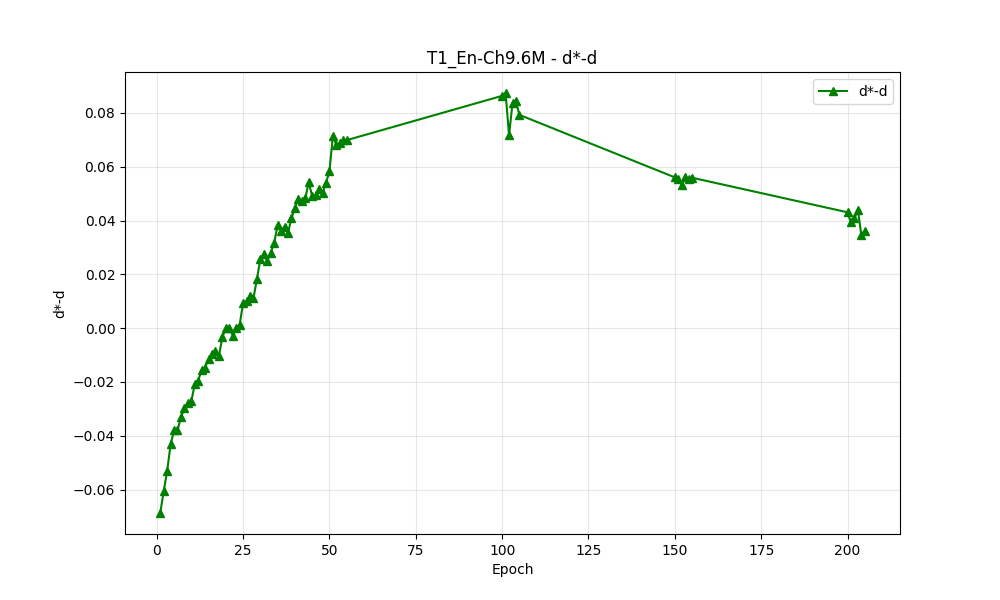}
		\caption{}
	\end{subfigure}
	\caption{Evolutionary Trajectory of the Early-Stopping Metric for \textbf{T1\_En-Ch 9.60M}. (a) describes the variation of d and d* with training epochs for the three models, respectively; while (b) illustrates the changes in the difference between d* and d across training epochs for these three models, respectively.}
	\label{fig:11}
\end{figure}

The evolutionary trajectories of HT indicator across the experimental configuration T1\_En-Ch 9.60M are illustrated in the figure~\ref{fig:11}, the corresponding plots for T2\_En-Ch 9.62M and T3\_En-Ch 19.20M are provided in \nameref{Appendix} \ref{sec:Evolutionary_Trajectory}.
Observation of the evolutionary trajectory of the HT indicator \( d^* - d \)
in Figure~\ref{fig:11} shows that, after an initial period of intense adjustment, the indicator rises rapidly, signaling the model’s entry into the Heavy-Tailed phase. Subsequently, the indicator reaches its peak and may exhibit minor fluctuations, suggesting that the Heavy-Tailed characteristics, after attaining their maximum intensity, can slightly diminish during the later stages of training.
\begin{table}[htbp]
	\centering
	\caption{Early-Stopping Loss, Final Loss (Epoch 205) \& Relative Difference Across Models}
	\label{tab:early_stopping_validation}
	\begin{tabular}{lcccccc}
		\toprule
		\textbf{Model} & \multicolumn{1}{c}{\textbf{ES Loss}} & \multicolumn{1}{c}{\textbf{Final Loss}} & \multicolumn{1}{c}{\textbf{Re.Diff Loss}}\\
		\midrule
		T1 & 2.3812 &  2.3465 & 1.479\%  \\
		T2 & 2.4360 & 2.4033  & 1.361\% \\
		T3 & 2.1039 &  2.0248  & 3.907\% \\
		\bottomrule
	\end{tabular}
\end{table}

Applying the proposed early-stopping criterion identifies the optimal epochs for the three experimental configurations as 100, 100, and 50, respectively. As detailed in Table~\ref{tab:early_stopping_validation}, which compares validation performance at these check-points against the final epoch, the strategy proves highly effective: for experiments T1\_En-Ch9.60M and T2\_En-Ch9.62M, continuing training for an additional 105 epochs yields a marginal loss reduction of less than 1.5\%, while for T3\_En-Ch19.20M, extending training by 155 epochs results in less than a 4\% improvement. These findings demonstrate that the spectral distance-based criterion can substantially reduce training time and computational costs while preserving near-final model performance with negligible degradation.

\section{Conclusion}\label{sec:conclusion}
The spectral distribution of \texttt{en.0.s.a.v} of Transformer models is profoundly influenced by training dynamics and exhibits persistent Heavy-Tailed (HT) characteristics during the prolonged phase of stable evolution. 
We investigate these phenomena from multiple perspectives. We analyze the Power-Law (PL) fitting parameters and their dynamic evolution, revealing that changes in these parameters effectively reflect the progress of model training.
 We partition the training process into three phases and demonstrate that the strong HT phase (Phase II) constitutes a critical period for effective fine-tuning in Transformer models. This finding offers a novel perspective for understanding the internal training mechanisms of LLMs. We propose a quantitative criterion for recognizing HT behavior via PL fitting, which can track the formation dynamics of HT structures and provide a spectral-based early-stopping criterion for Transformer models. Notably, this approach enables stopping decisions without reliance on a separate validation set.

These findings also raise several compelling questions for future exploration. For example, an important direction is to evaluate whether the proposed spectral early-stopping criterion generalizes to other LLMs such as LLaMA-3 and DeepSeek-MoE. Systematic assessment of the method’s generalizability and limitations represents a crucial avenue for subsequent research. Investigating these underlying mechanisms and extending the criterion will provide significant directions for future work.
\bibliography{example_paper}
\newpage
\appendix
\onecolumn
\section*{Appendix}\label{Appendix}
\section{Theoretical Analysis of the Optimization Dynamics of the Value Matrix $W^V$}
\label{sec: Choose V matrix}

Recent studies \citep{gurbuzbalaban2021heavy} have demonstrated that heavy-tailed behavior in deep neural networks is closely linked to the optimization dynamics induced by gradient descent. In particular, \citet{song2024unraveling} proved that in a single-layer Transformer, optimizing only the value matrix ($W^V$) is sufficient to achieve the global minimum. This finding underscores the dominant and simplifying role of $W^V$ during model training.In practical Transformer training, however, all three attention matrices ($W^Q, W^K, W^V$) are updated jointly. Due to the non-convexity introduced by the softmax operation, the optimization landscapes associated with $W^Q$ and $W^K$ are considerably more intricate than that of $W^V$. To support our empirical observation that the dynamics of $W^V$ exhibit smooth, nearly linear convergence alongside heavy-tailed statistical behavior, we extend existing convergence analyses—originally established for quadratic losses—to the cross-entropy loss employed in our tasks. Specifically, we analyze the optimization dynamics of $W^V$ while holding $W^Q$ and $W^K$ fixed.

Our objective is not to establish global convergence for the entire Transformer model, but rather to demonstrate that, under mild assumptions, optimizing $W^V$ alone follows provably stable gradient dynamics with sublinear convergence rates.We formalize this theoretical guarantee in the following proposition:

\begin{proposition}
	\label{prop:wv_convergence}
	
	Under the setting where only the Value matrix $W^V$ is optimized, the loss function $f(W^V)$ satisfies:
	\begin{enumerate}
		\item \textbf{Convexity:} $f$ is convex with respect to $W^V$.
		\item \textbf{Smoothness:} $\nabla_{W^V} f$ is $L$-Lipschitz continuous for some finite constant $L > 0$.
		\item \textbf{Global Convergence:} For a learning rate $\eta \le 1/L$, gradient descent converges sublinearly:
		\begin{equation}
			f(W^V_t) - f(W^{V*}) \le \frac{\|W^V_0 - W^{V*}\|_F^2}{2\eta t} = \mathcal{O}(1/t),
		\end{equation}
		where $W^{V*}$ is an optimal solution with $f(W^{V*}) = 0$.
	\end{enumerate}
\end{proposition}

In the following subsections, we provide a detailed analysis to support Proposition~\ref{prop:A1}. We begin by formalizing the problem setup in Section~A.1. Subsequently, we verify the three claims of the proposition sequentially: Section~A.2 establishes the \textbf{convexity}, Section~A.3 proves the \textbf{smoothness} condition, and Section~A.4 derives the \textbf{global convergence rate}.

\subsection{Setup and Notation}

We consider a single-layer attention block with the standard cross-entropy loss:
\begin{equation}
	f(W^V) = \frac{1}{N} \sum_{i=1}^{N} \ell(z_i, y_i),
	\quad 
	\ell(z_i, y_i) = -\log(\mathrm{softmax}(z_i)_{y_i}),
\end{equation}
where $z_i \in \mathbb{R}^V$ are the logits corresponding to the $i$-th token, and $y_i$ is the ground-truth label.

For each input sequence $X$, the attention mechanism produces the logits:
\begin{equation}
	z_i = \mathcal{A}_i(W^V) = \mathrm{softmax}\left( \frac{Q_i K^\top}{\sqrt{d}} \right) V_i,
\end{equation}
where $Q_i = X_i W^Q$, $K = X W^K$, $V = X W^V$.  
When $W^Q$ and $W^K$ are fixed, the mapping $W^V \mapsto z_i$ is linear.  
We denote this linear operator by $\mathcal{A}: W^V \mapsto z$.

Thus the full objective can be written compactly as:
\begin{equation}
	f(W^V) = \mathcal{L}_{CE}(\mathcal{A}(W^V), y),
\end{equation}
where $\mathcal{L}_{CE}$ is convex in its first argument.

The update rule under gradient descent is:
\begin{equation}
	W^V_{t+1} = W^V_t - \eta \nabla_{W^V} f(W^V_t),
\end{equation}
where $\eta > 0$ is the learning rate.

\subsection{Convexity Proof}

We first establish that $f(W^V)$ is convex with respect to $W^V$.
\begin{lemma}
	(Convexity of Cross-Entropy) For a single example with logits $z$ and label $y$, the loss
	\[
	\ell(z, y) = -\log(\mathrm{softmax}(z)_y)
	\]
	is convex in $z$.
\end{lemma}
\textit{Proof.}
Defining that  $p_k = \mathrm{softmax}(z)_k = \frac{e^{z_k}}{\sum_j e^{z_j}}$.  
The gradient is
\[
\frac{\partial \ell}{\partial z_k} = p_k - \delta_{k y},
\]
and the Hessian is
\[
H_{kj} = \frac{\partial^2 \ell}{\partial z_k \partial z_j}
= p_k (\delta_{kj} - p_j)
= \mathrm{diag}(p) - p p^\top.
\]
For any vector $v$,
\[
v^\top H v = \sum_k p_k v_k^2 - \Big(\sum_k p_k v_k\Big)^2 \ge 0,
\]
by the Cauchy–Schwarz inequality. Hence, $H$ is positive semidefinite and $\ell(z, y)$ is convex in $z$.  
\hfill $\square$

\begin{remark}
	The Hessian satisfies $H \mathbf{1} = 0$ (translation invariance), hence $\ell$ is convex but not strongly convex.
\end{remark}

Since $\mathcal{A}$ is a linear operator in $W^V$, the composition $\ell(\mathcal{A}(W^V), y)$ remains convex.  
By summation over samples, $f(W^V)$ is convex in $W^V$.

\subsection{Smoothness Proof}

We next prove that $f(W^V)$ has Lipschitz-continuous gradients.

By the chain rule:
\begin{equation}
	\nabla_{W^V} f(W^V) = \mathcal{A}^*\big(\nabla_z \mathcal{L}_{CE}(z(W^V))\big),
\end{equation}
where $\mathcal{A}^*$ is the adjoint operator (matrix transpose in finite dimensions).

Since $\nabla_z \mathcal{L}_{CE}(z) = p - y_{\mathrm{onehot}}$, the gradient w.r.t. $W^V$ depends linearly on $\mathcal{A}^*$ and the residual $(p - y)$.

\begin{lemma}(Smoothness)
	The function $\mathcal{L}_{CE}(z)$ is 1/4-smooth in $z$, i.e.,
	\[
	\|\nabla_z \mathcal{L}_{CE}(z_1) - \nabla_z \mathcal{L}_{CE}(z_2)\|_2
	\le \tfrac{1}{4}\|z_1 - z_2\|_2.
	\]
\end{lemma}
\textit{Proof Sketch.}
The Hessian of $\mathcal{L}_{CE}$ is $H = \mathrm{diag}(p) - p p^\top$.  
All eigenvalues of $H$ lie in $[0, 1/4]$, as proved in standard logistic regression analysis.  
Hence, the gradient is 1/4-Lipschitz.  
\hfill $\square$

By composition with linear operator $\mathcal{A}$,

\begin{align*}
	\|\nabla_{W^V} f(W_1^V) - \nabla_{W^V} f(W_2^V)\|_F
	&= \|\mathcal{A}^*(\nabla_z \mathcal{L}_{CE}(z_1) - \nabla_z \mathcal{L}_{CE}(z_2))\|_F \\
	&\le \|\mathcal{A}^*\|_{\mathrm{op}} \tfrac{1}{4}
	\|\mathcal{A}\|_{\mathrm{op}} \|W_1^V - W_2^V\|_F.
\end{align*}

Thus $f$ is $L$-smooth with
\[
L = \tfrac{1}{4} \|\mathcal{A}\|_{\mathrm{op}} \|\mathcal{A}^*\|_{\mathrm{op}}.
\]

\subsection{Convergence Rate Proof}

With convexity and $L$-smoothness established, we can apply standard convergence results.

\begin{proposition}[Convergence of $W^V$ Optimization]
	\label{prop:WV_convergence}
	Let $f(W^V)$ be convex and $L$-smooth.  
	Then, for learning rate $\eta \le 1/L$, gradient descent satisfies:
	\[
	f(W^V_t) - f(W^{V*})
	\le \frac{\|W^V_0 - W^{V*}\|_F^2}{2\eta t}
	= \mathcal{O}(1/t).
	\]
\end{proposition}

\textit{Proof.}
The standard descent lemma for $L$-smooth functions gives:
\[
f(W_{t+1}^V) \le f(W_t^V)
- \eta \|\nabla f(W_t^V)\|_F^2
+ \frac{L\eta^2}{2}\|\nabla f(W_t^V)\|_F^2.
\]

For $\eta \le 1/L$,
\[
f(W_{t+1}^V) \le f(W_t^V)
- \frac{\eta}{2}\|\nabla f(W_t^V)\|_F^2.
\]

By convexity, $f(W_t^V) - f(W^{V*}) \le \langle \nabla f(W_t^V), W_t^V - W^{V*} \rangle$.  
Combining and telescoping the inequality yields:
\[
f(W_t^V) - f(W^{V*}) \le \frac{\|W_0^V - W^{V*}\|_F^2}{2\eta t}.
\]
\hfill $\square$

\subsection{Discussion}

This analysis demonstrates that when $W^Q$ and $W^K$ are approximately fixed, the optimization of $W^V$ reduces to a convex and $L$-smooth problem, ensuring an $\mathcal{O}(1/t)$ convergence rate. This explains the empirically observed stability of $W^V$ dynamics and reinforces our conclusion that the heavy-tailed phenomenon manifests most prominently in the $V$ matrix.

\section{Pervasive Heavy-Tailed Features and Analysis of Representative Matrices}
\subsection{Analysis of Representative Matrices}\label{sec: heavy-tail}\label{sec:choose en.0.s.a.v}
In this section, we rigorously verify the heavy-tailed characteristics of Transformer weight matrices by comparing the goodness-of-fit of the Power-Law distribution against alternative hypotheses. As noted by \citet{clauset2009power} and \citet{alstott2014powerlaw}, even when a dataset appears to follow a power-law distribution, alternative distributions such as the exponential or log-normal may provide an equally good or even superior fit. To rigorously test this, one can compare the goodness-of-fit among competing models and compute the p-value associated with the power-law hypothesis \citep{clauset2009power, klaus2011statistical, alstott2014powerlaw}.

We compared the fits of the Power-Law, Log-Normal, and Exponential distributions to the empirical spectral distributions (ESDs) of the Q, K, and V matrices. In the English–Chinese dataset, we frequently observed that the Log-Normal distribution provided a better fit than the Power-Law—for example, fore.g., for ``\texttt{de.0.a.v}'', ``\texttt{de.1.a.v
}'', ``\texttt{de.0.s.a.q}'', ``\texttt{de.0.s.a.v}'', ``\texttt{en.1.s.a.k}'', ``\texttt{en.0.s.a.v}''\footnote{Here, `de' denotes the decoder layer, `en' denotes the encoder layer, `0',`1',`2' indicate the layer indices, `a' represents attention, `s.a' refers to the self-attention layer, and `q', `k', and `v' correspond to the Q, K, V matrix.} at epoch=1. \citet{alstott2014powerlaw} emphasized that when a plausible generative mechanism exists for the log-normal distribution, it often fits empirical data as well as--or better than--the power-law. Similarly, \citet{clauset2009power} observed that, unless the dataset is exceptionally large, no statistical test can reliably distinguish between the log-normal and power-law fits. Nonetheless, prior work \citep{martin2021predicting} has shown that for most weight matrices, both the bulk and the tail of their ESDs can be well described by a power-law distribution. Therefore, we consider the power-law model a reasonable and theoretically grounded choice.

Despite this, we conducted direct comparisons between the power-law and log-normal distributions and obtained results consistent with those of \citet{clauset2009power} and \citet{alstott2014powerlaw}. Detailed results are omitted due to space constraints. Because the exponential distribution serves as a minimal benchmark for identifying heavy-tailed behavior \citep{alstott2014powerlaw}, we further compared the power-law and exponential fits under various parameter settings using the same datasets. By examining the p-values and log-likelihood ratio R, we identified matrices for which the power-law fit significantly outperformed the exponential fit, indicating clear heavy-tailed characteristics. A summary of these results is provided in Table \ref{tab:heavy_tail_stats_wide}.

\begin{table}[H]
	\centering
	\caption{Statistically Significant Heavy-Tailed Matrices}
	\label{tab:heavy_tail_stats_wide}
	\begin{tabular}{@{}ll@{}}
		\toprule
		\textbf{Setup} & \textbf{Significantly Heavy-Tailed Matrices} \\
		\midrule
		\textbf{T1\_En-Ch9.6M} & en.0.s.a.v (50); de.0.s.a.q (153); en.0.s.a.q(50) \\
		\addlinespace	
		\textbf{T2\_En-Ch9.6M} & en.0.s.a.v (50); de.0.s.a.k (100); de.0.s.a.q (200) \\
		\addlinespace
		\textbf{T3\_En-Ch19.2M} & en.0.s.a.v (50); en.0.s.a.q (50); en.1.s.a.q (50) \\
		\bottomrule
	\end{tabular}
\end{table}
The results in Table~\ref{tab:heavy_tail_stats_wide} reveal a particularly consistent pattern: the \texttt{en.0.s.a.v} weight matrix persistently exhibits a superior Power-Law (PL) fit compared to the Exponential fit, regardless of the dataset or architectural configuration. To further investigate this robust heavy-tailed behavior, we tracked the evolution of the p-value and the log-likelihood ratio ($R$) for this matrix across training epochs. The temporal dynamics under different experimental settings are visualized in the figures below. Specifically, we identify the condition where the p-value is $< 0.1$ and the log-likelihood ratio is positive ($R > 0$) as indicative of pronounced heavy-tailed characteristics. Based on this criterion, the trends observed in Figures~\ref{fig:image1}, \ref{fig:image2} lead to the final conclusion that the \texttt{en.0.s.a.v} matrix consistently exhibits significant heavy-tailed characteristics throughout the training process.
\begin{figure}[H] 
	\centering 
	\includegraphics[width=0.7\textwidth]{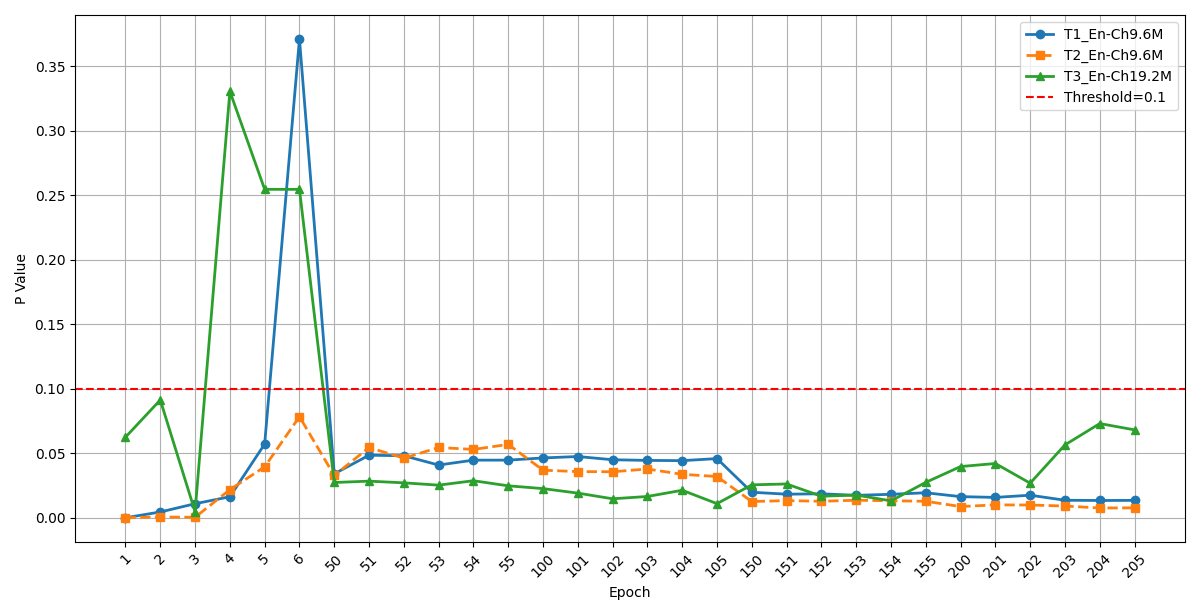}
	\caption{P-value of the power-law fit for the `en.0.s.a.v' matrix versus training epoch across different experimental setups.}
	\label{fig:image1}
\end{figure}
\begin{figure}[h!]
	\centering
	\includegraphics[width=0.7\textwidth]{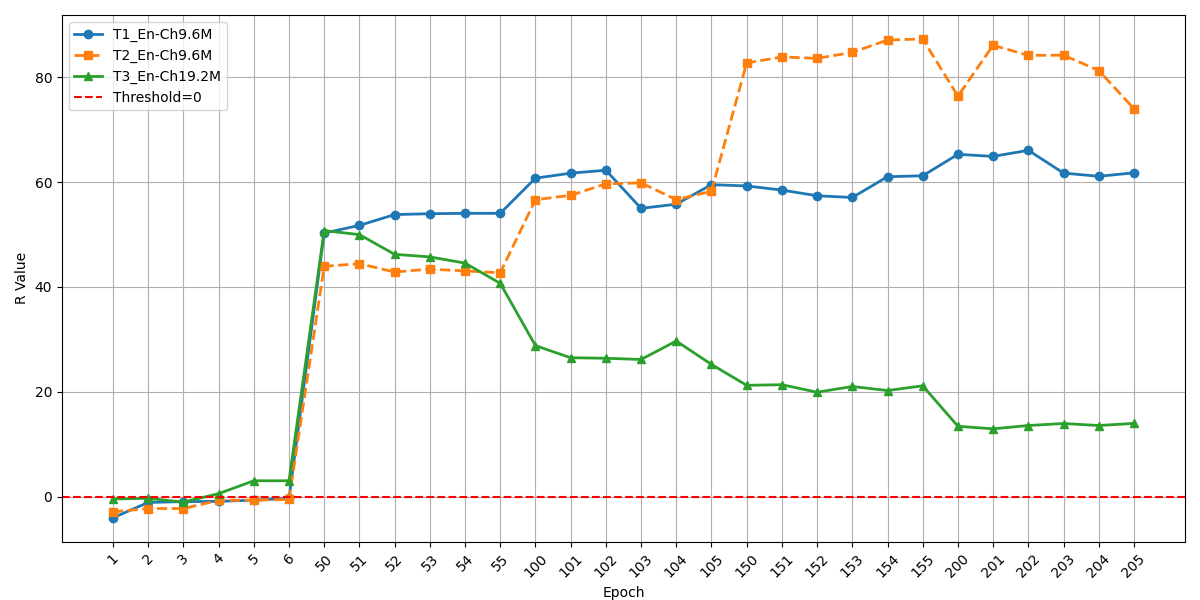}
	\caption{Log-likelihood ratio (R) for the power-law vs. exponential fit for the \texttt{en.0.s.a.v} matrix versus training epoch. }
	\label{fig:image2}
\end{figure}

\subsection{Supplementary experiments with three training phases}\label{supplementary_three_phases}
\begin{figure}[t]
	\centering	
	\begin{subfigure}[b]{1\linewidth}
		\centering
		\includegraphics[width=0.7\linewidth]{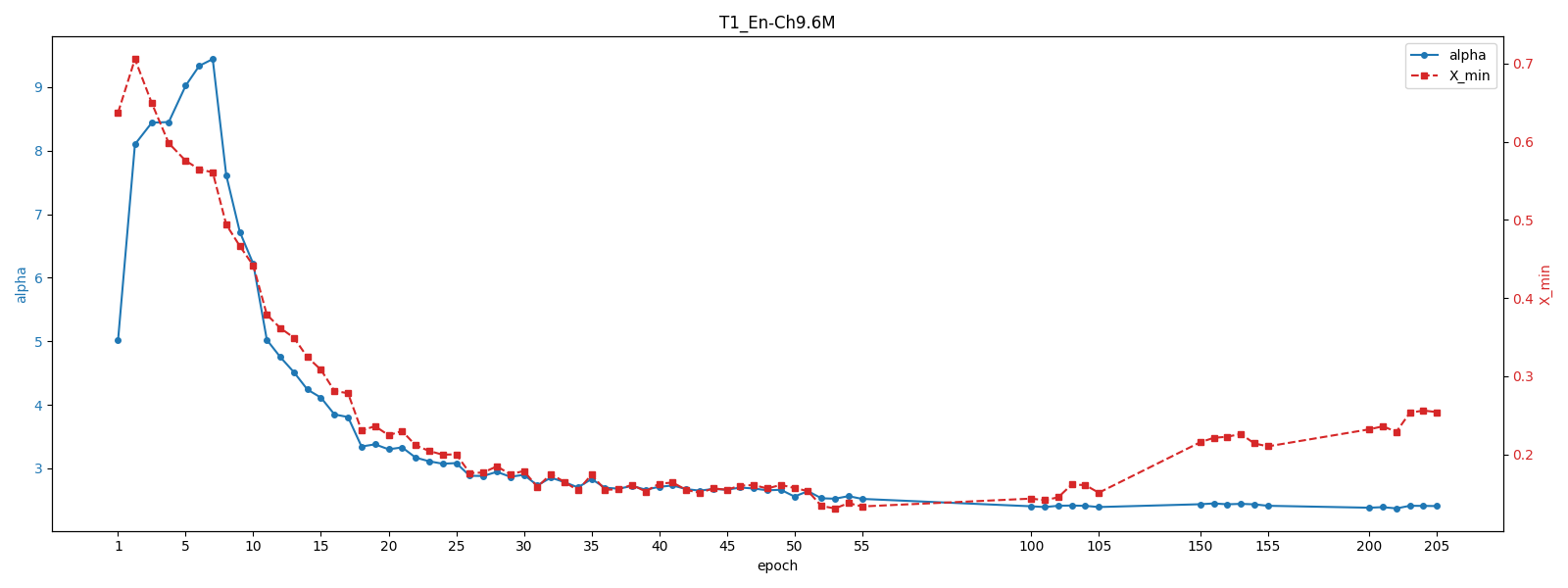}
		\caption{}
	\end{subfigure}
	
	\begin{subfigure}[b]{1\linewidth}
		\centering
		\includegraphics[width=0.7\linewidth]{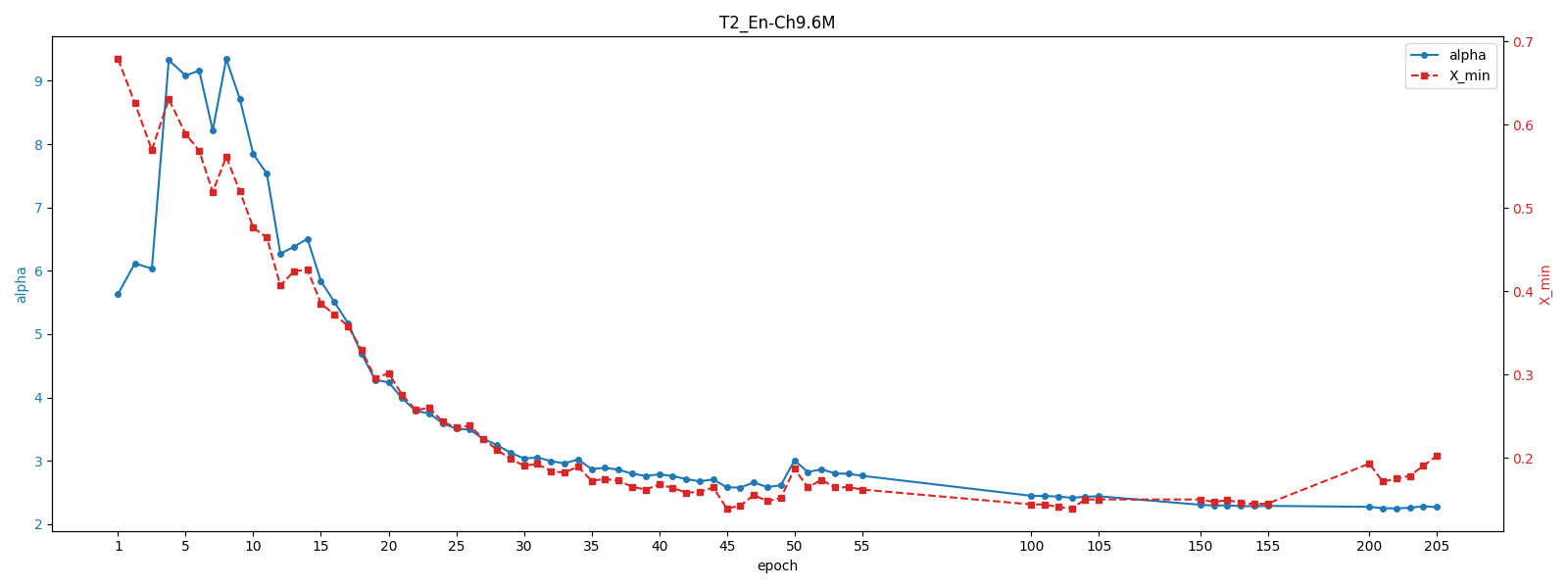}
		\caption{}
	\end{subfigure}
	\caption{Evolution of $\alpha$ and $x_{\text{min}}$ for \texttt{en.0.s.a.v}  matrix tracked across T1\_En-Ch9.60M, T2\_En-Ch9.62M. }
	\label{fig:alpha xmin for en.0.s.a.v in appendix}
\end{figure}

Figure~\ref{fig:alpha xmin for en.0.s.a.v in appendix} illustrates the detailed trajectories of the fitted parameters ($\alpha$ and $x_{\min}$) for the \texttt{en.0.s.a.v} matrix as a function of training epochs, specifically under the experimental configurations \textbf{T1\_En-Ch9.6M} and \textbf{T2\_En-Ch9.6M}.

To further substantiate the correspondence between these internal spectral dynamics and external learning performance, Figure~\ref{fig:loss_accuracy_curves} depicts the evolution of Loss and Accuracy across all training epochs for the three experimental setups. Given that classic early stopping is typically implemented by monitoring validation set performance, this comparison is essential. During \textbf{Phase I}, the model exhibits the most rapid performance improvement, characterized by a sharp decline in loss and a swift increase in accuracy. In \textbf{Phase II}, the rate of improvement decelerates, and the curves smoothen as the model undergoes steady learning and refinement. Finally, in \textbf{Phase III}, the performance curves plateau, with both validation loss and accuracy showing negligible further gains. This saturation aligns consistently with the microscopic phenomena observed in the spectral tail.
\begin{figure}[t]
	\centering
	\begin{subfigure}[b]{1\linewidth}
		\centering
		\includegraphics[width=0.5\linewidth]{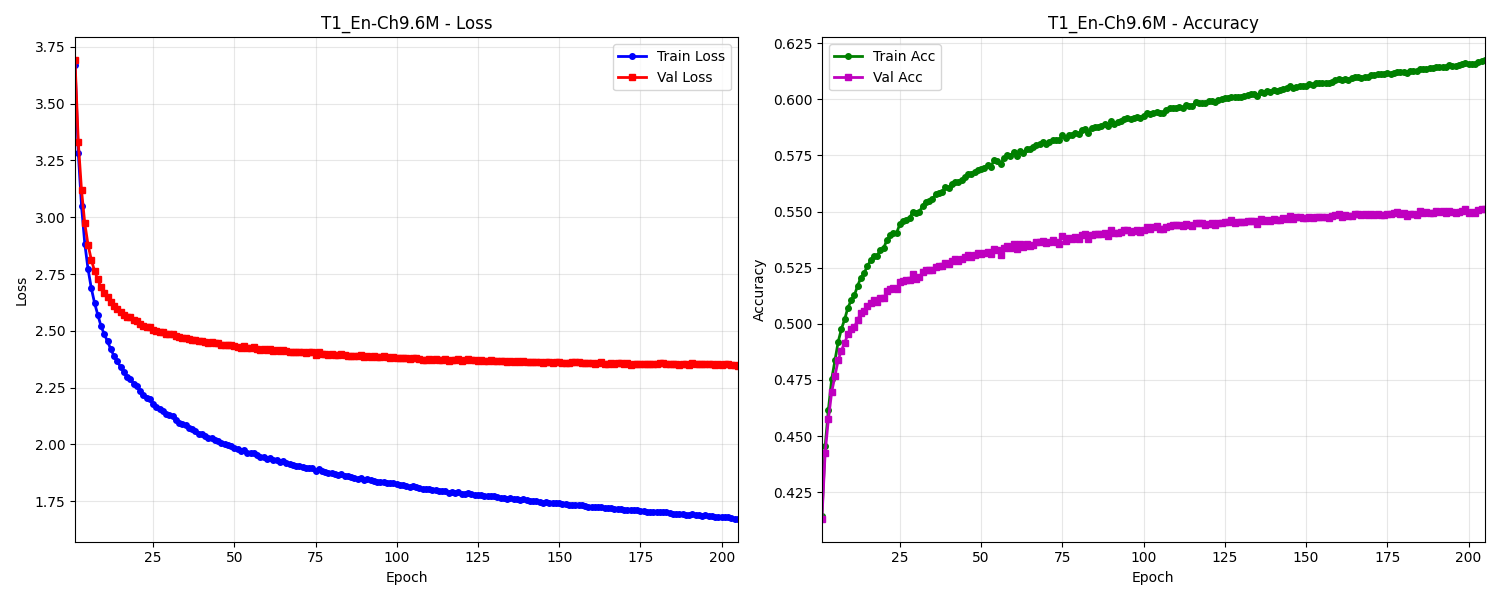}
		\caption{}
	\end{subfigure}
	
	\begin{subfigure}[b]{1\linewidth}
		\centering
		\includegraphics[width=0.5\linewidth]{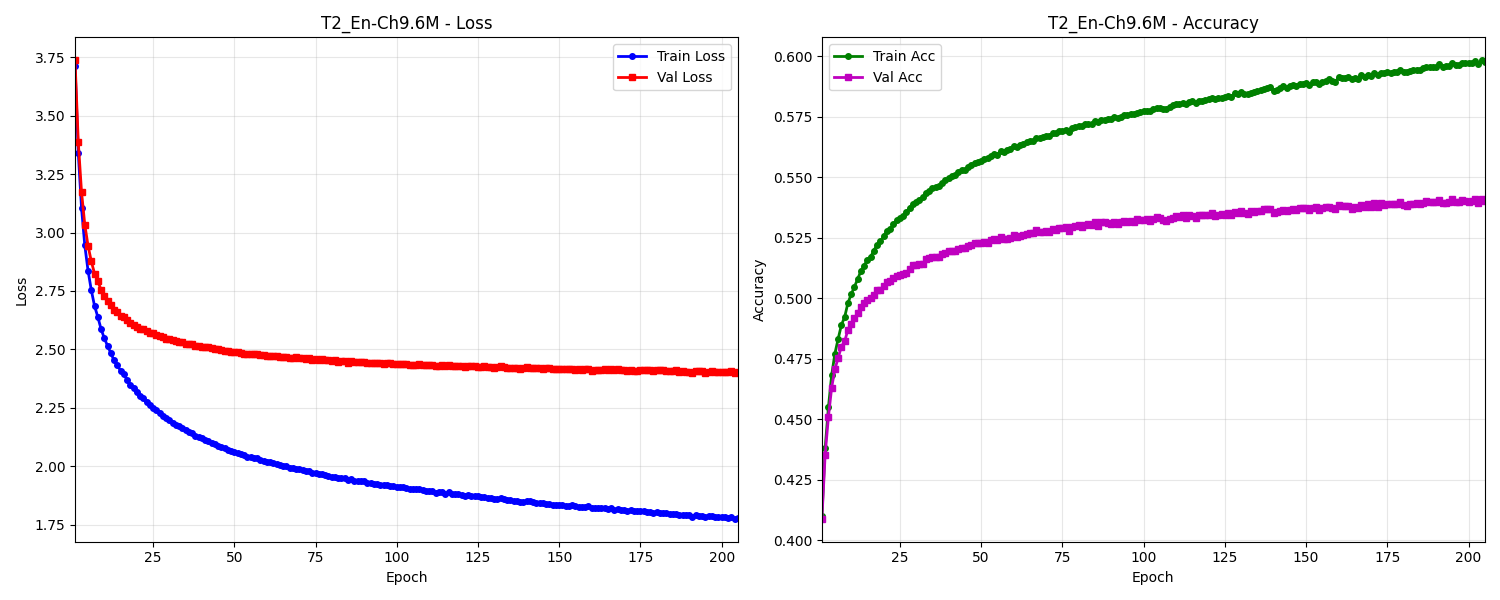}
		\caption{}
	\end{subfigure}
	
	\begin{subfigure}[b]{1\linewidth}
		\centering
		\includegraphics[width=0.5\linewidth]{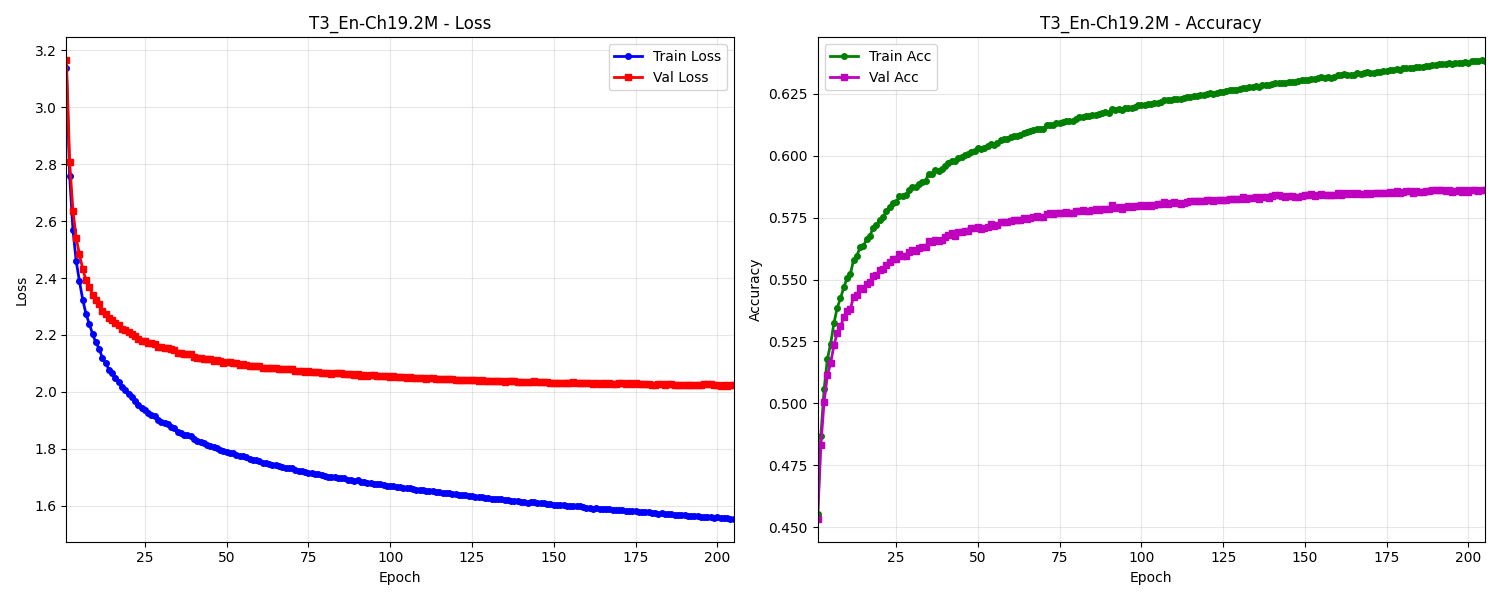}
		\caption{}
	\end{subfigure}
	\caption{Evolution of Loss values and Accuracy values with training epochs for the three experimental setups.}
	\label{fig:loss_accuracy_curves}
\end{figure}

\section{Dynamic Evolution of Spectral Parameters and Training Phase Division}\label{sec:en.0.s.a.v more stable}

This section analyzes the dynamic evolution of spectral parameters ($\alpha$ and $x_{\min}$) in Transformer weight matrices to uncover the intrinsic dynamics across different training phases. By establishing a comprehensive workflow involving data extraction, SVD decomposition, and Power-Law fitting, combined with multi-dimensional visualization techniques, we systematically evaluate the stability of various attention mechanisms. Our analysis reveals three key findings: (1) shallow layers are significantly more stable than deep layers; (2) the cross-attention mechanism exhibits the most pronounced fluctuations; and (3) the shallow encoder $V$ matrix (\texttt{en.0.s.a.v}) demonstrates the most consistent convergence trend and a unique heavy-tailed evolution pattern.(This result aligns with the representative matrix analysis presented in Appendix~\ref{sec: heavy-tail}.) 

On a log-log plot, a Power-Law distribution appears as a straight line, where the parameter $\alpha$ corresponds to the negative slope. The value of $\alpha$ directly characterizes the distribution's ``tail heaviness'': a larger $\alpha$ implies a steeper slope and a faster decay. To capture training dynamics, we followed this workflow:
\begin{enumerate}
	\item \textbf{Data Collection:} During the training of the two experimental setups (T1\_En-Ch, and T2\_En-Ch), we periodically extracted the Q, K and V weight matrices from all attention mechanisms at checkpoints at multiple training epochs. Specifically, matrices were extracted every 50 epochs, resulting in six sampled epochs per experiment.
	\item \textbf{Matrix Construction:} For each Q, K and V matrix (denoted as $\mathbf{W}$), we constructed the characteristic matrix $\mathbf{X} = \mathbf{W}^T\mathbf{W}$ and applied Singular Value Decomposition (SVD) to obtain its spectral distribution.
	\item \textbf{Parameter Fitting:} A power-law fit was applied to each spectral distribution. For every matrix at each sampling point, we recorded two parameters: $\alpha$ and $x_{min}$, where $x_{min}$ represents the minimum boundary for which PL fitting can be performed, and $\alpha$ denotes the negative value of the slope of the function obtained after PL fitting.
	\item \textbf{Data Coverage:} Since each Transformer layer contains distinct attention mechanisms (encoder self-attention, decoder self-attention, and cross-attention), we tracked all Q, K, and V matrices across layers. For example, in the 3-layer T1 model, we collected spectral parameter data from nine attention modules per checkpoint—three encoder self-attention, three decoder self-attention, and three decoder cross-attention layers.
\end{enumerate}

To perform a comprehensive multi-dimensional analysis of the dynamic evolution of the Power-Law fitting parameters ($\alpha$ and $x_{\min}$), we developed a set of four complementary visualization techniques. For each of the two experimental configurations, these methods are presented in subplots (a)–(d) of the corresponding result figures. The objective is to capture the macroscopic trends, local stability, and fine-grained fluctuations of the parameters throughout training. The visualization results for the two experimental setups are shown in Figures~\ref{fig: alpha for T1_En_Ch 9.60M}, \ref{fig: xmin for T1_En_Ch 9.60M}, \ref{fig: alpha for T2_En_Ch 9.62M}, and \ref{fig: xmin for T2_En_Ch 9.62M}.
\begin{itemize}
	\item \textbf{Raw Dynamic Observation (Subplots a \& c):} Directly visualizes the unprocessed fitted parameters over training epochs to show immediate fluctuations.
	\item \textbf{Macroscopic Trend Analysis (Subplot b):} Plots the mean values of parameters over predefined epoch groups, illustrating the long-term trajectory of parameter evolution.
	\item \textbf{Local Stability Analysis (Subplot d):} Depicts the standard deviation of parameters within sampling groups; larger deviations indicate instability, while smaller ones suggest convergence.
\end{itemize}

\begin{figure}[h!]
	\centering
	\begin{subfigure}[b]{0.37\linewidth}
		\centering
		\includegraphics[height=0.7\linewidth]{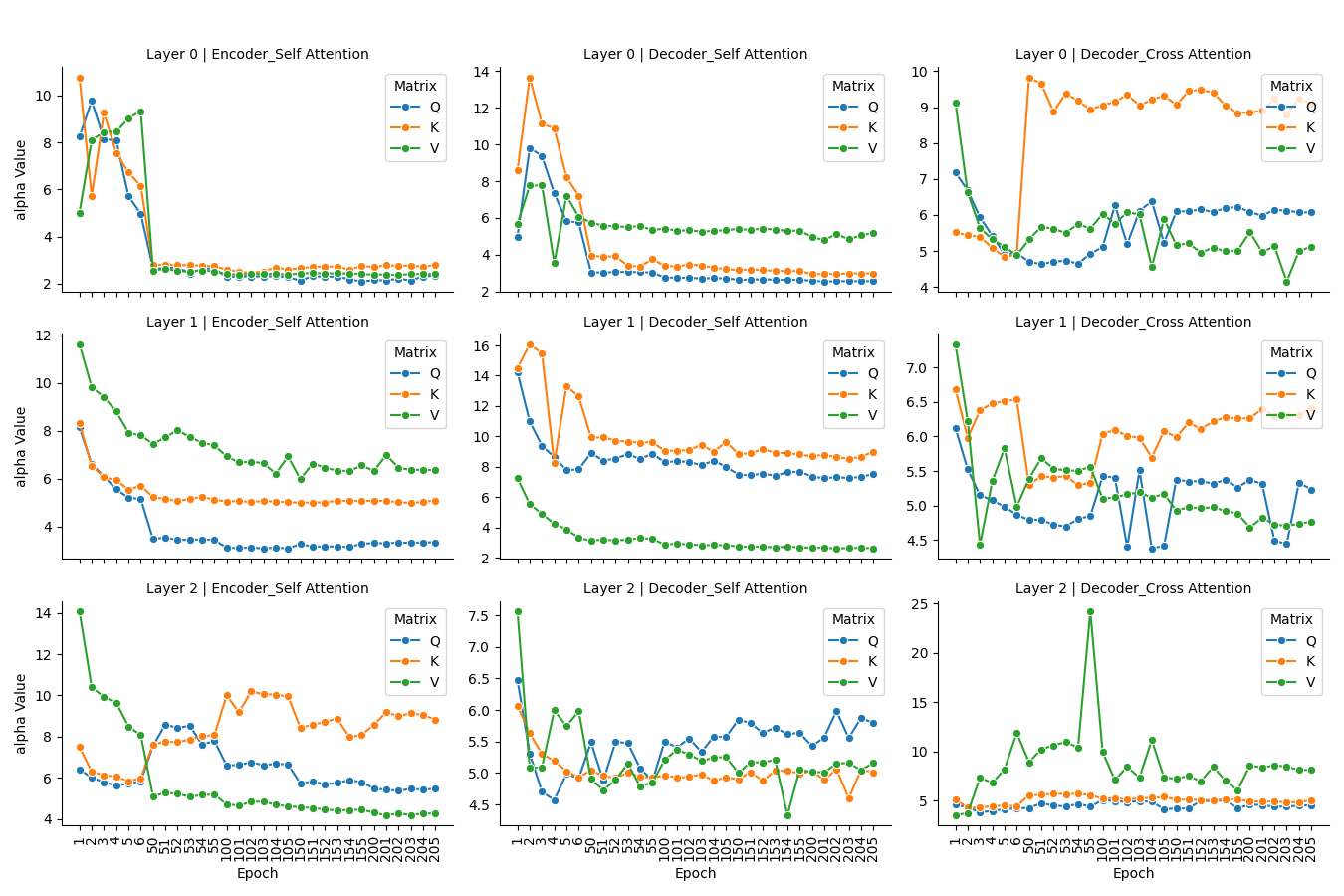}
		\caption{}
	\end{subfigure}
	\hspace{0.1\linewidth}
	\begin{subfigure}[b]{0.37\linewidth}
		\centering
		\includegraphics[height=0.7\linewidth]{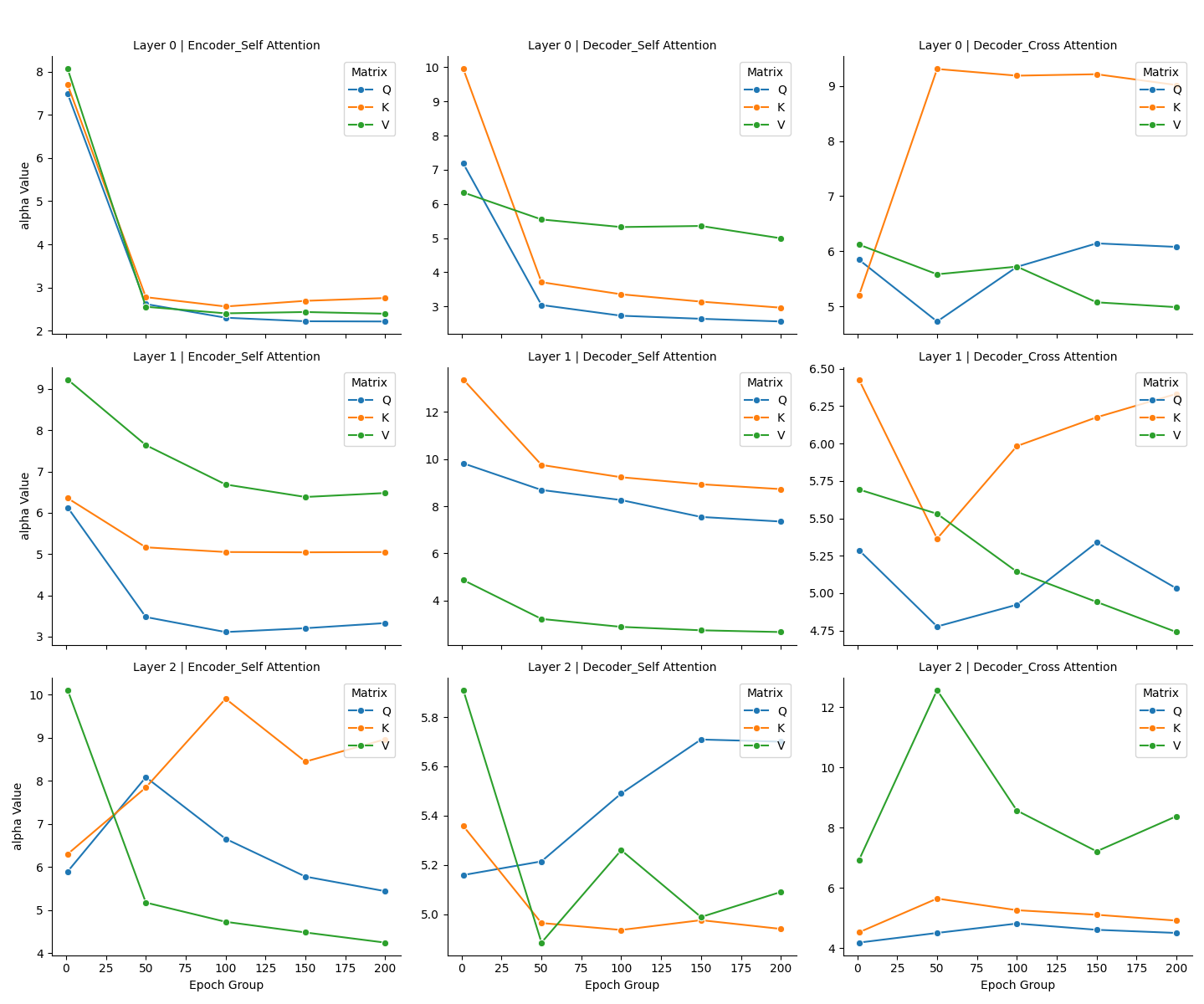}
		\caption{}
	\end{subfigure}
	
	\begin{subfigure}[b]{0.37\linewidth}
		\centering
		\includegraphics[height=0.7\linewidth]{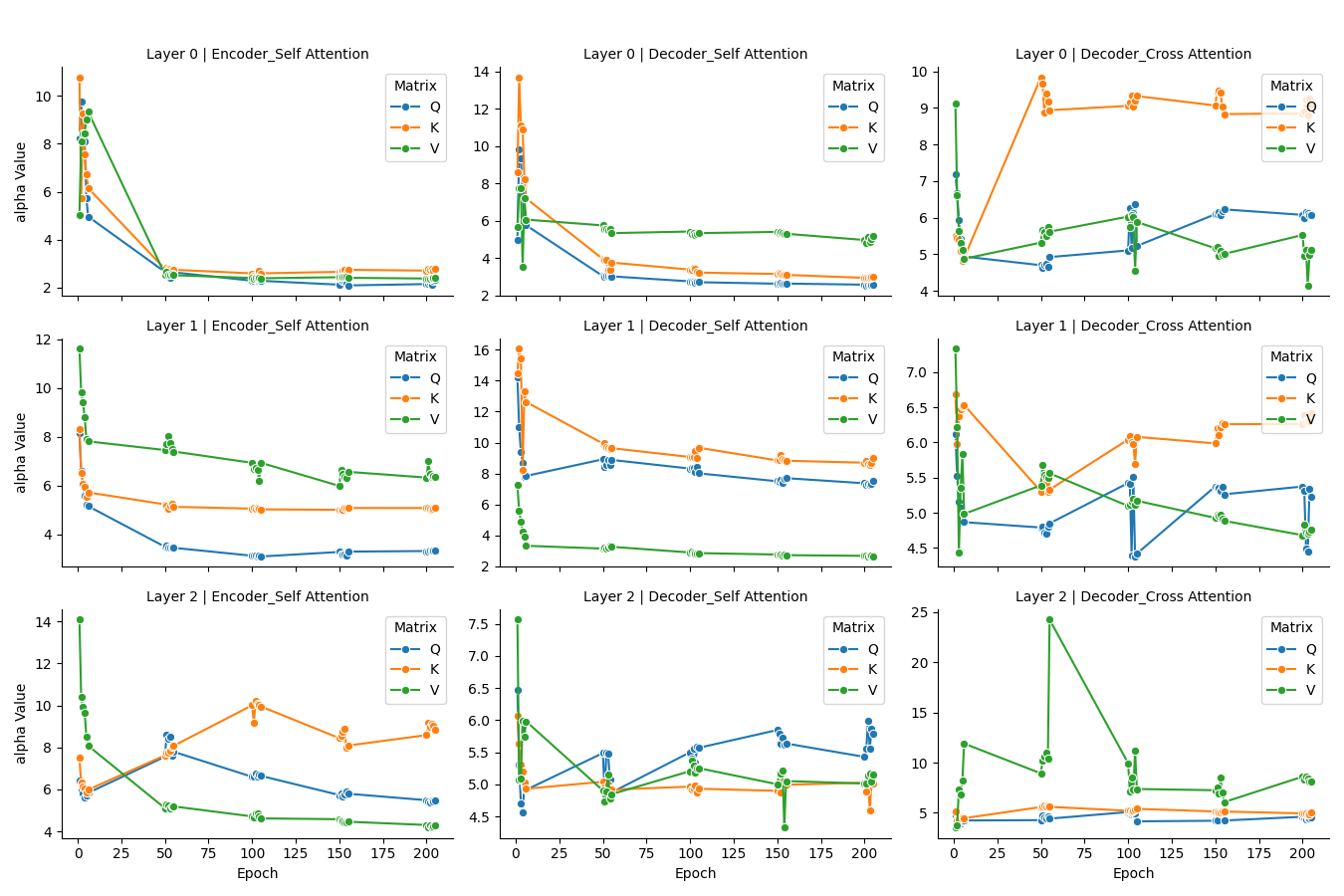}
		\caption{}
	\end{subfigure}
	\hspace{0.1\linewidth}
	\begin{subfigure}[b]{0.37\linewidth}
		\centering
		\includegraphics[height=0.7\linewidth]{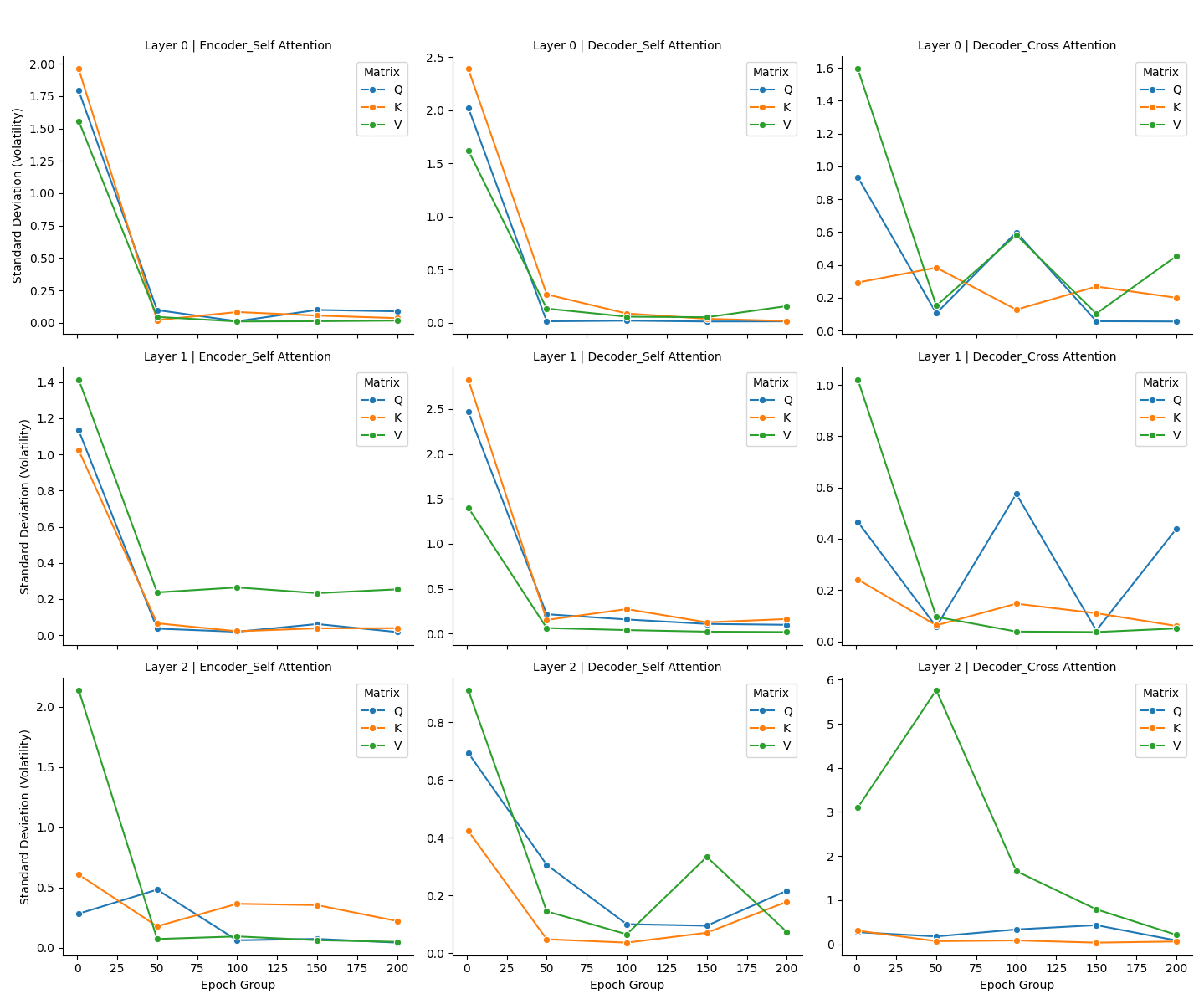}
		\caption{}
	\end{subfigure}
	\caption{ $\alpha$ parameter for T1\_En-Ch 9.60M}
	\label{fig: alpha for T1_En_Ch 9.60M}
\end{figure}
\begin{figure}[H]
	\centering
	\begin{subfigure}[b]{0.37\linewidth}
		\centering
		\includegraphics[height=0.7\linewidth]{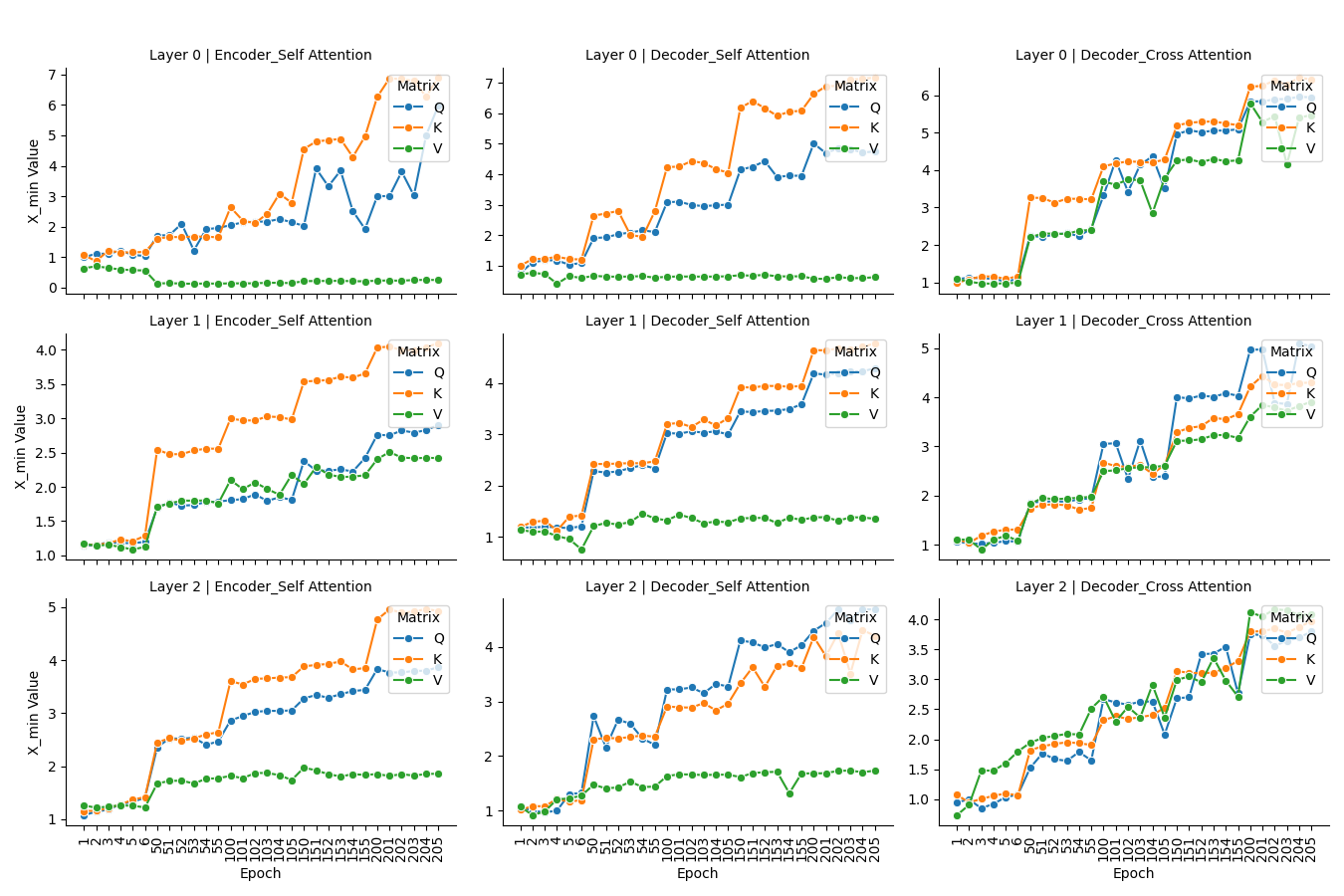}
		\caption{}
	\end{subfigure}
	\hspace{0.1\linewidth}
	\begin{subfigure}[b]{0.37\linewidth}
		\centering
		\includegraphics[height=0.7\linewidth]{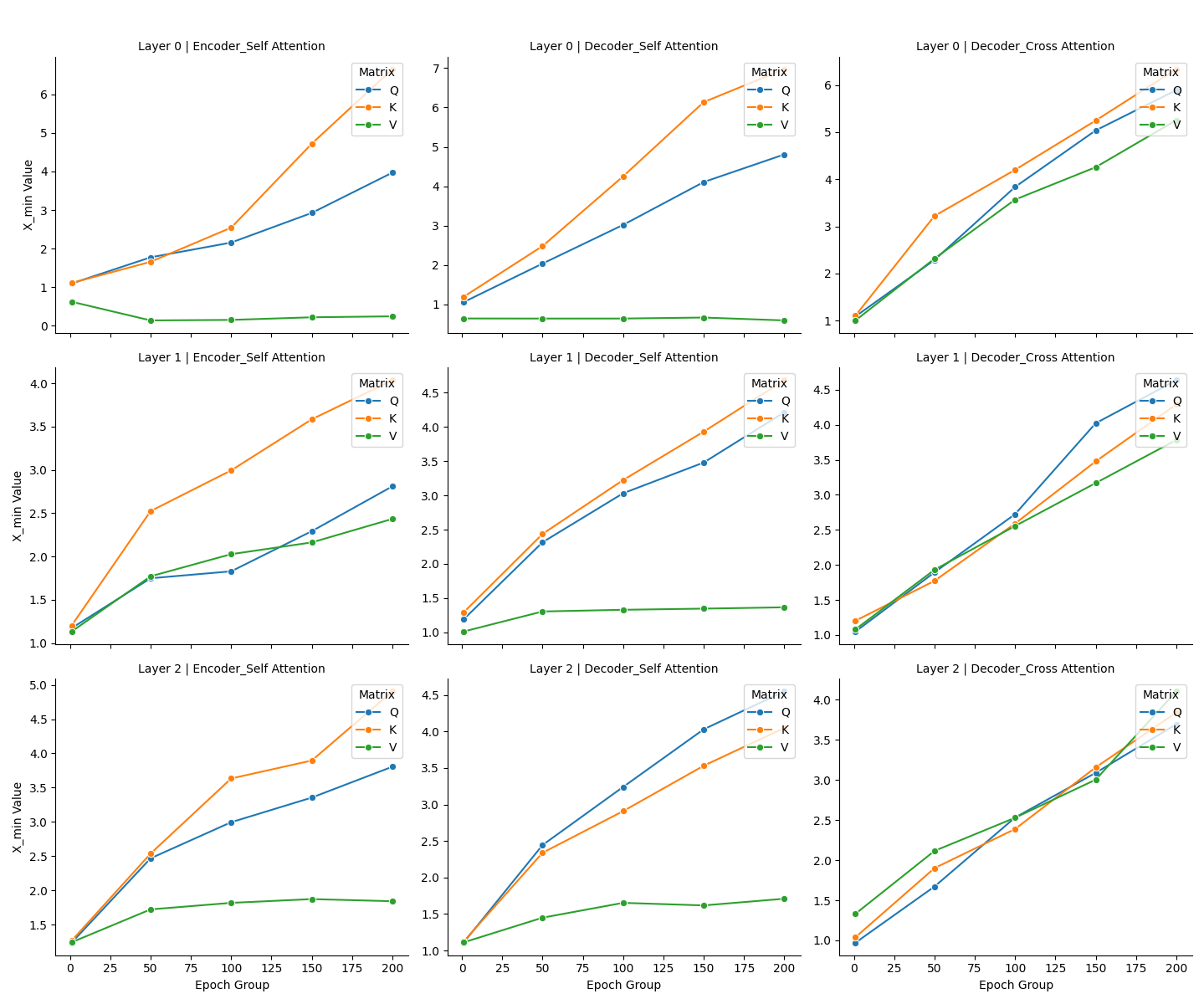}
		\caption{}
	\end{subfigure}
	
	\begin{subfigure}[b]{0.37\linewidth}
		\centering
		\includegraphics[height=0.7\linewidth]{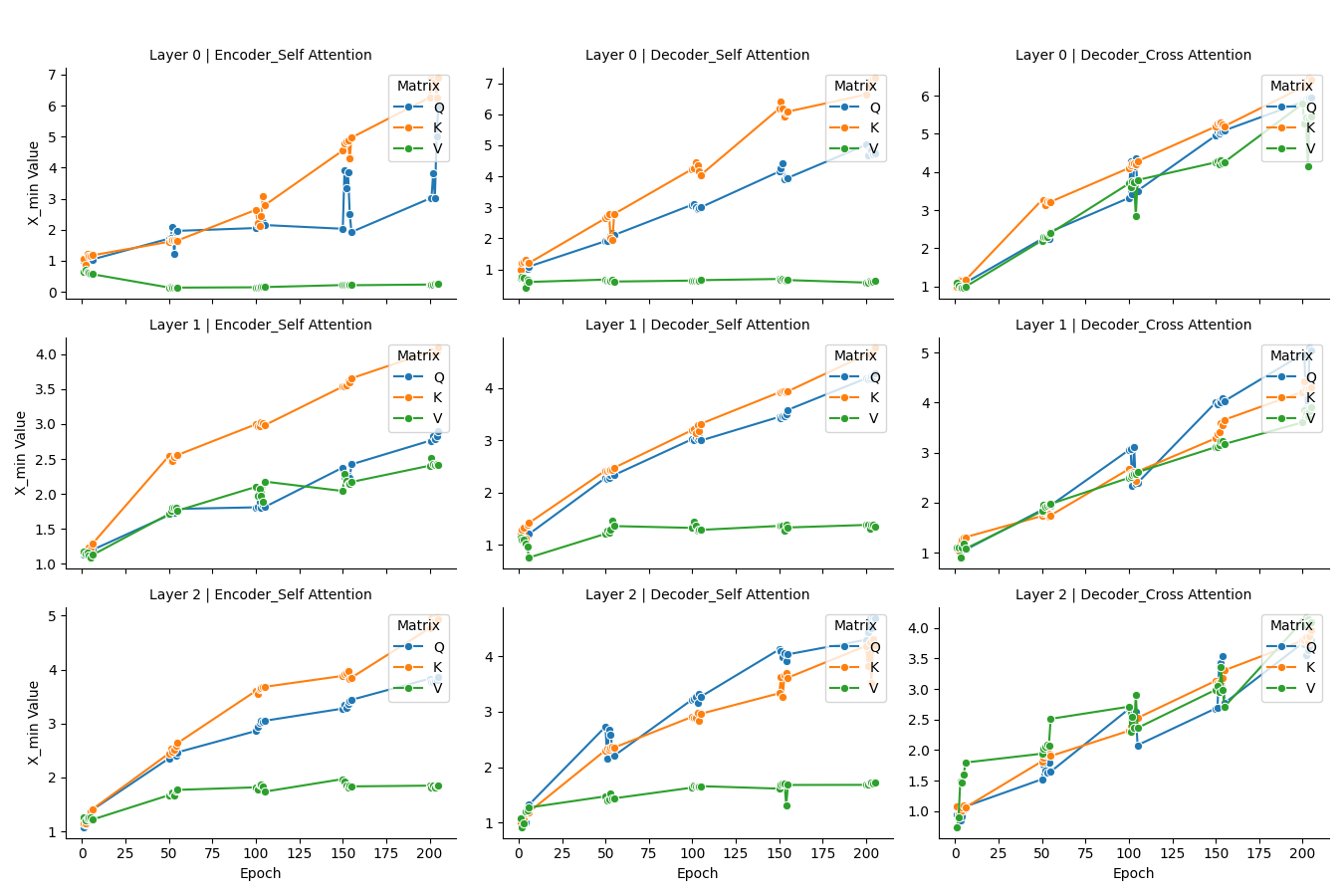}
		\caption{}
	\end{subfigure}
	\hspace{0.1\linewidth}
	\begin{subfigure}[b]{0.37\linewidth}
		\centering
		\includegraphics[height=0.7\linewidth]{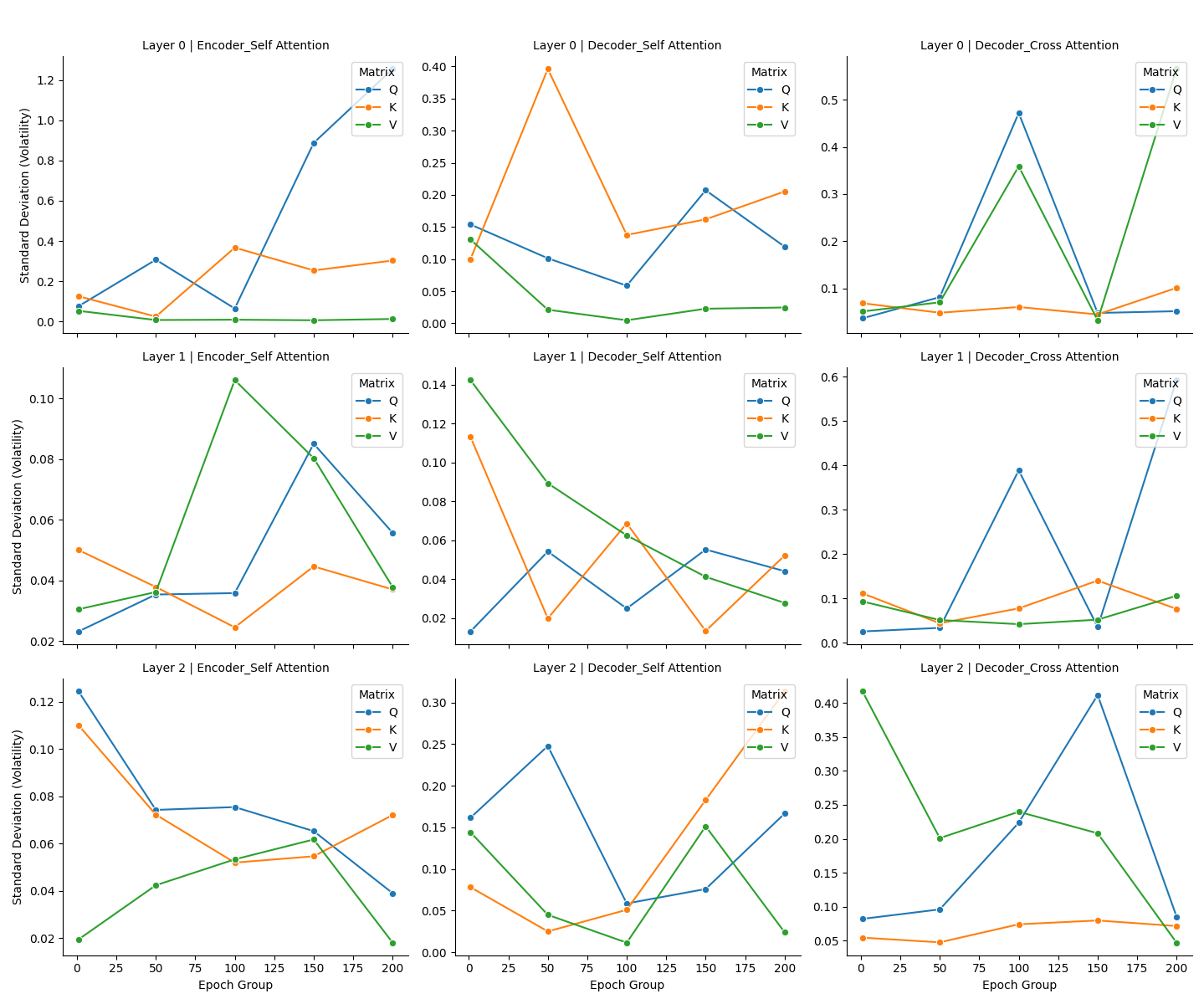}
		\caption{}
	\end{subfigure}
	\caption{$x_{min}$ parameter for T1\_En-Ch 9.60M}
	\label{fig: xmin for T1_En_Ch 9.60M}
\end{figure}
\begin{figure}[H]
	\raggedright
	\begin{subfigure}[b]{0.23\linewidth}
		\includegraphics[height=1\linewidth]{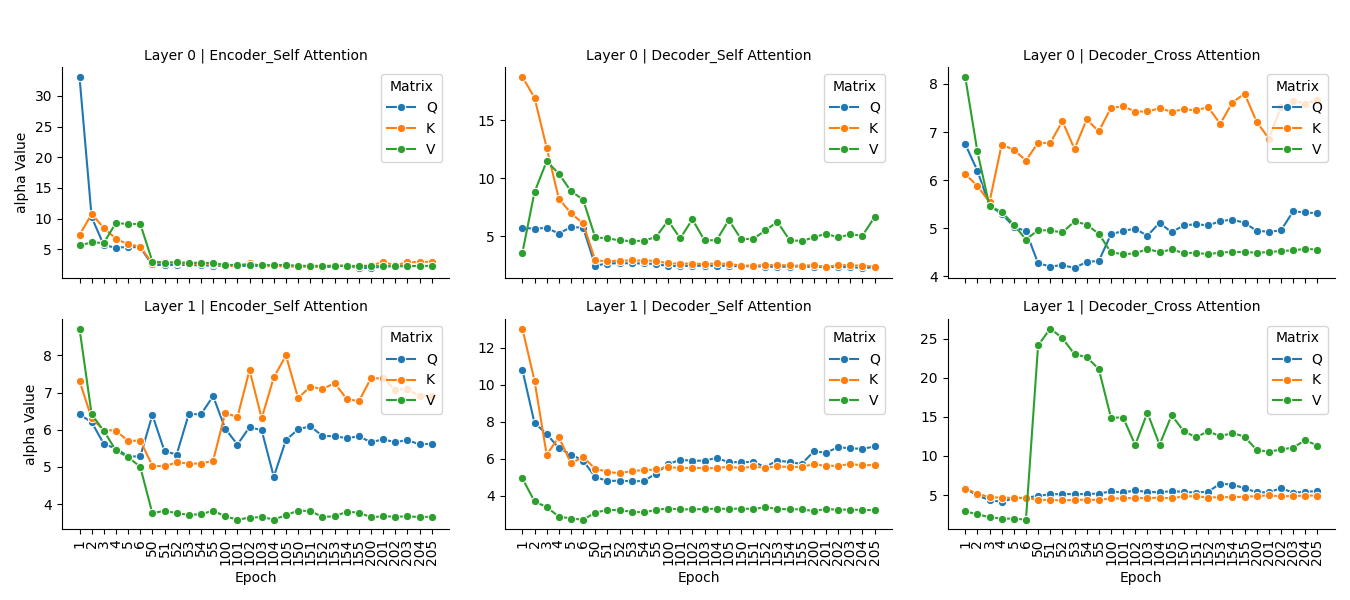}
		\caption{}
		\label{fig: xmin for T2_En-Ch 9.62M(a)}
	\end{subfigure}
	\hspace{0.32\linewidth}
	\begin{subfigure}[b]{0.23\linewidth}
		\includegraphics[height=1\linewidth]{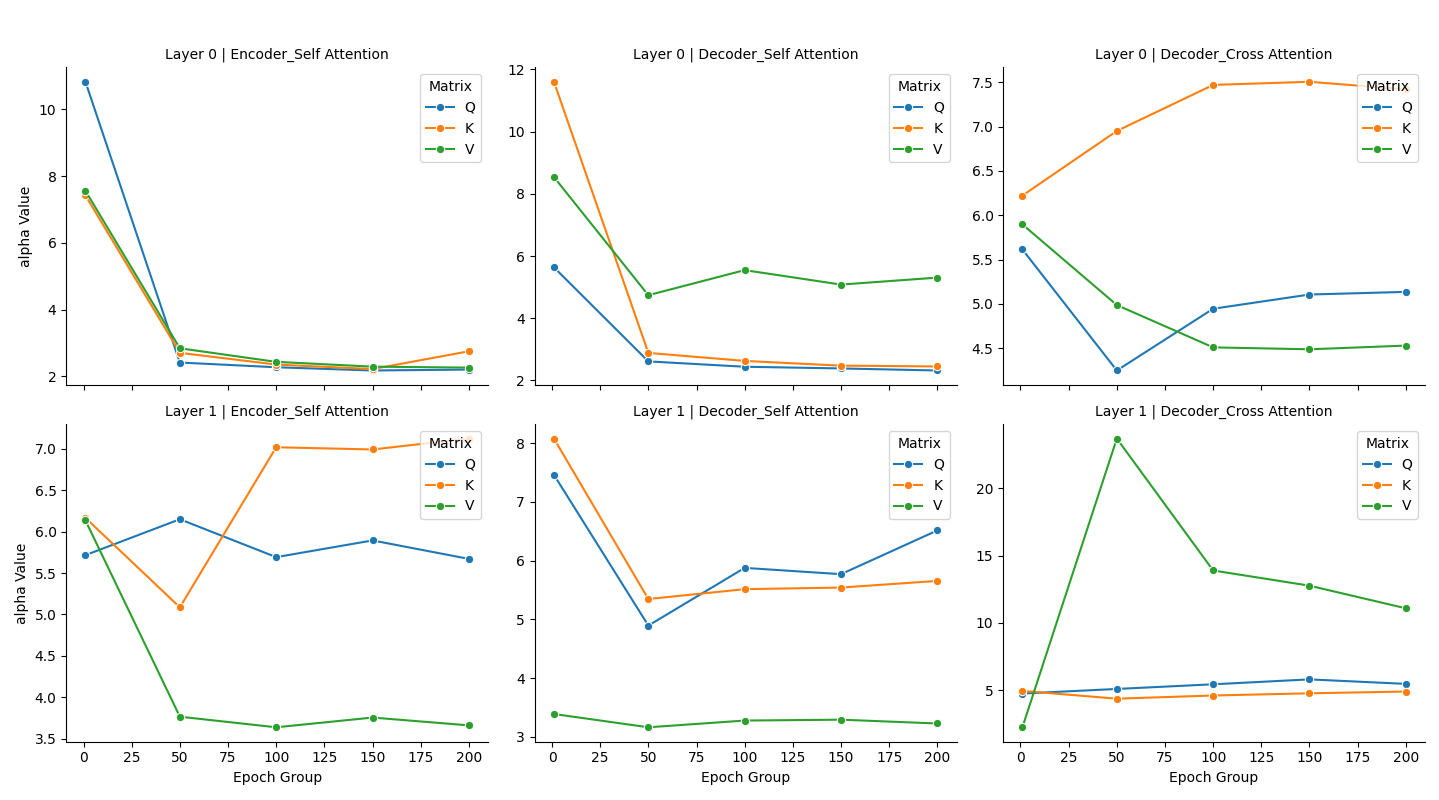}
		\caption{}
	\end{subfigure}
	
	\begin{subfigure}[b]{0.23\linewidth}
		\includegraphics[height=1\linewidth]{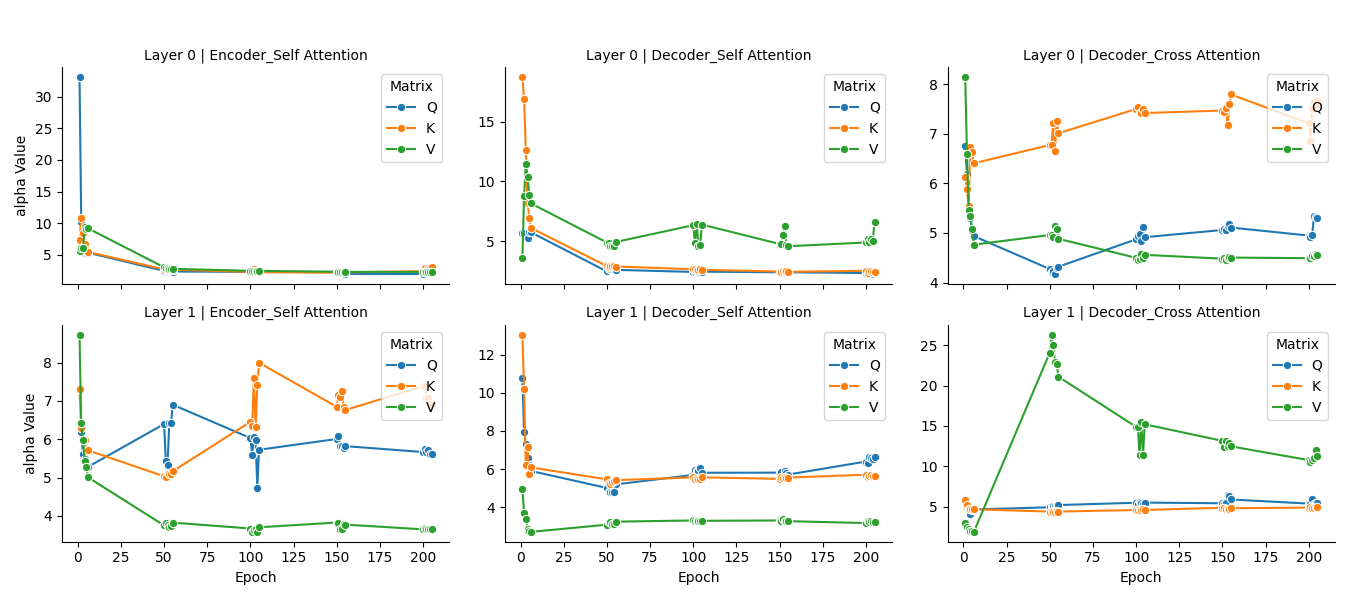}
		\caption{}
	\end{subfigure}
	\hspace{0.32\linewidth}
	\begin{subfigure}[b]{0.23\linewidth}
		\includegraphics[height=1\linewidth]{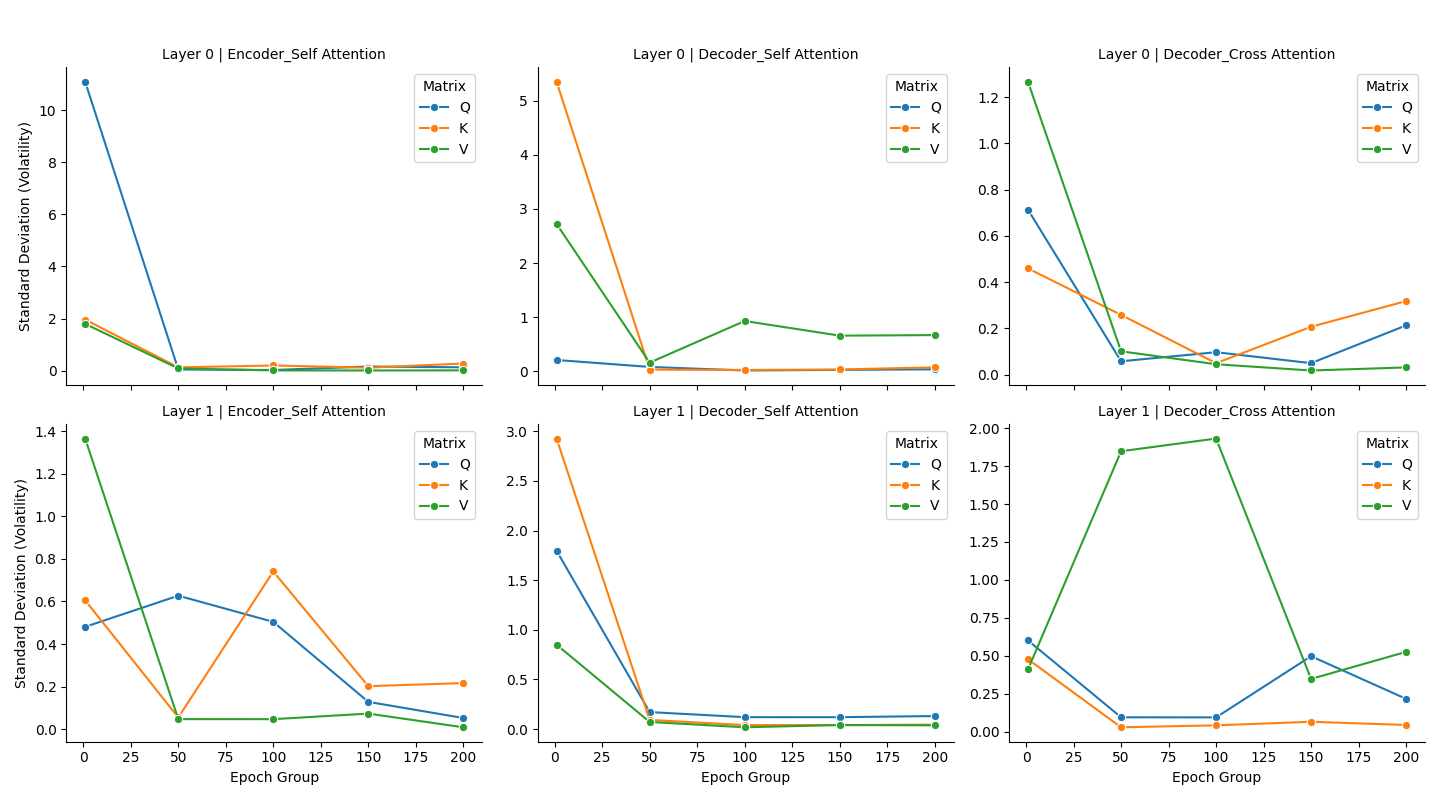}
		\caption{}
	\end{subfigure}
	\caption{$\alpha$ parameter for T2\_En-Ch 9.62M}
	\label{fig: alpha for T2_En_Ch 9.62M}
\end{figure}

\begin{figure}[H]
	\raggedright
	\begin{subfigure}[b]{0.23\linewidth}
		\centering
		\includegraphics[height=1\linewidth]{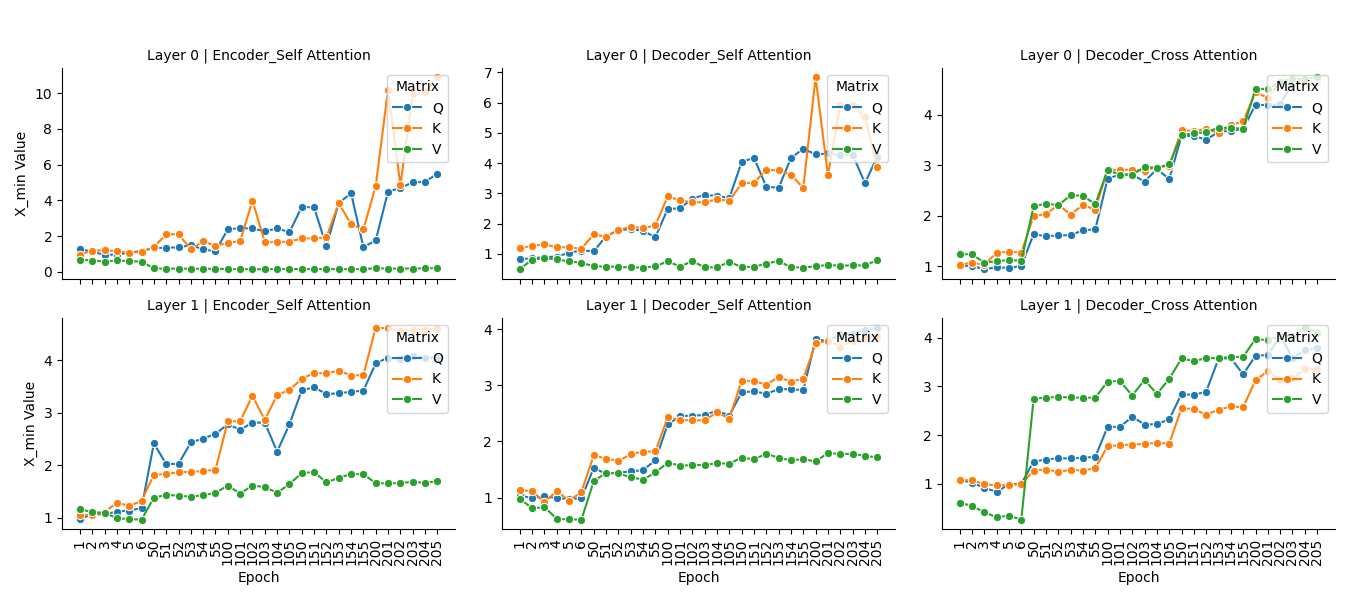}
		\caption{}
	\end{subfigure}
	\hspace{0.32\linewidth}
	\begin{subfigure}[b]{0.23\linewidth}
		\centering
		\includegraphics[height=1\linewidth]{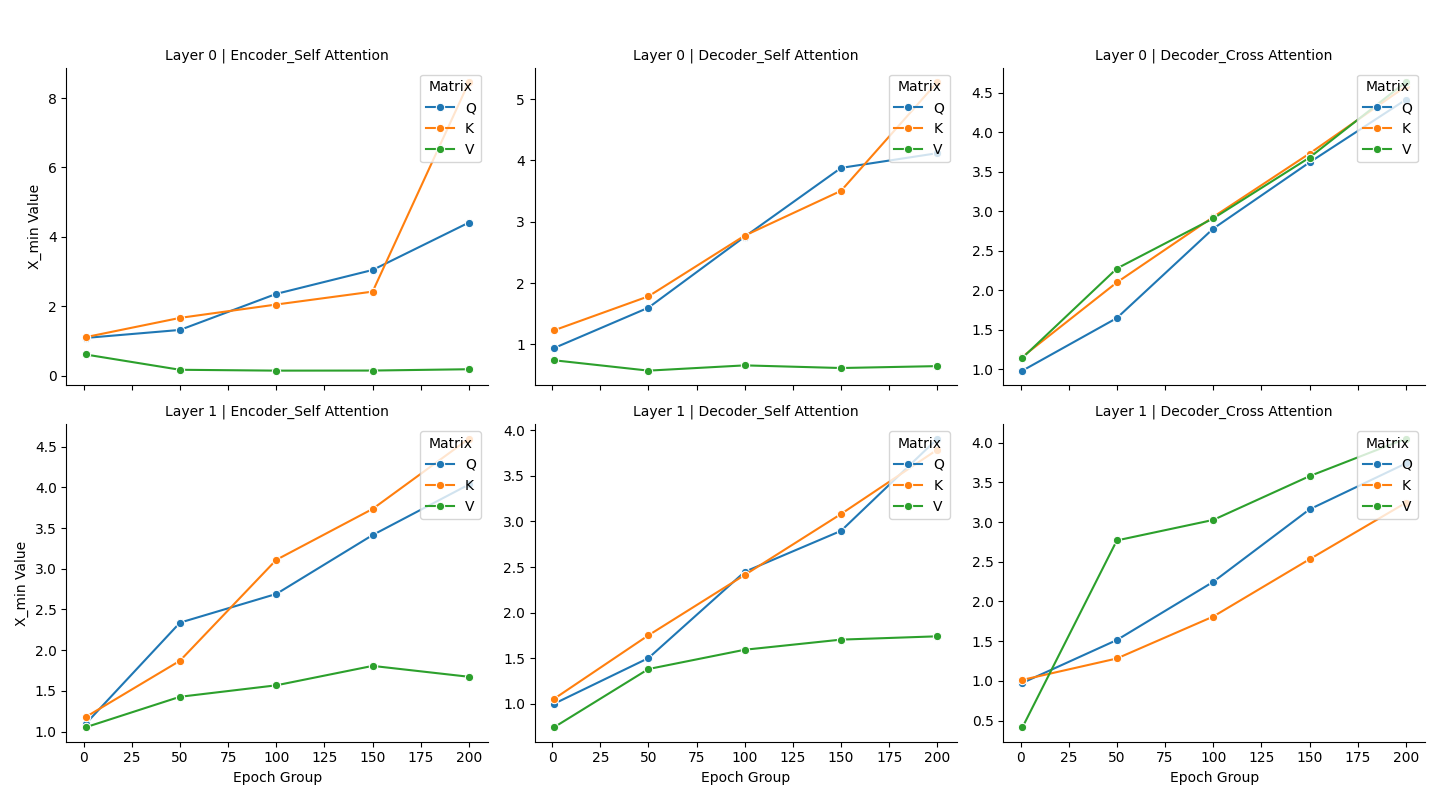}
		\caption{}
	\end{subfigure}
	
	\begin{subfigure}[b]{0.23\linewidth}
		\centering
		\includegraphics[height=1\linewidth]{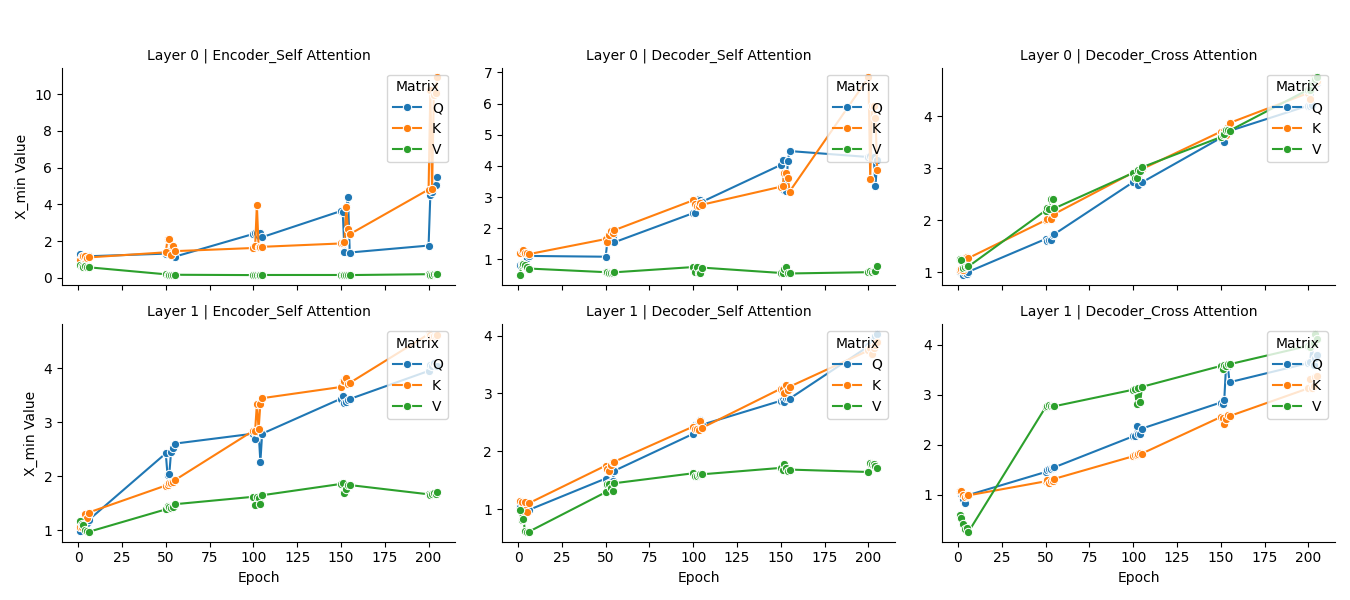}
		\caption{}
	\end{subfigure}
	\hspace{0.32\linewidth}
	\begin{subfigure}[b]{0.23\linewidth}
		\centering
		\includegraphics[height=1\linewidth]{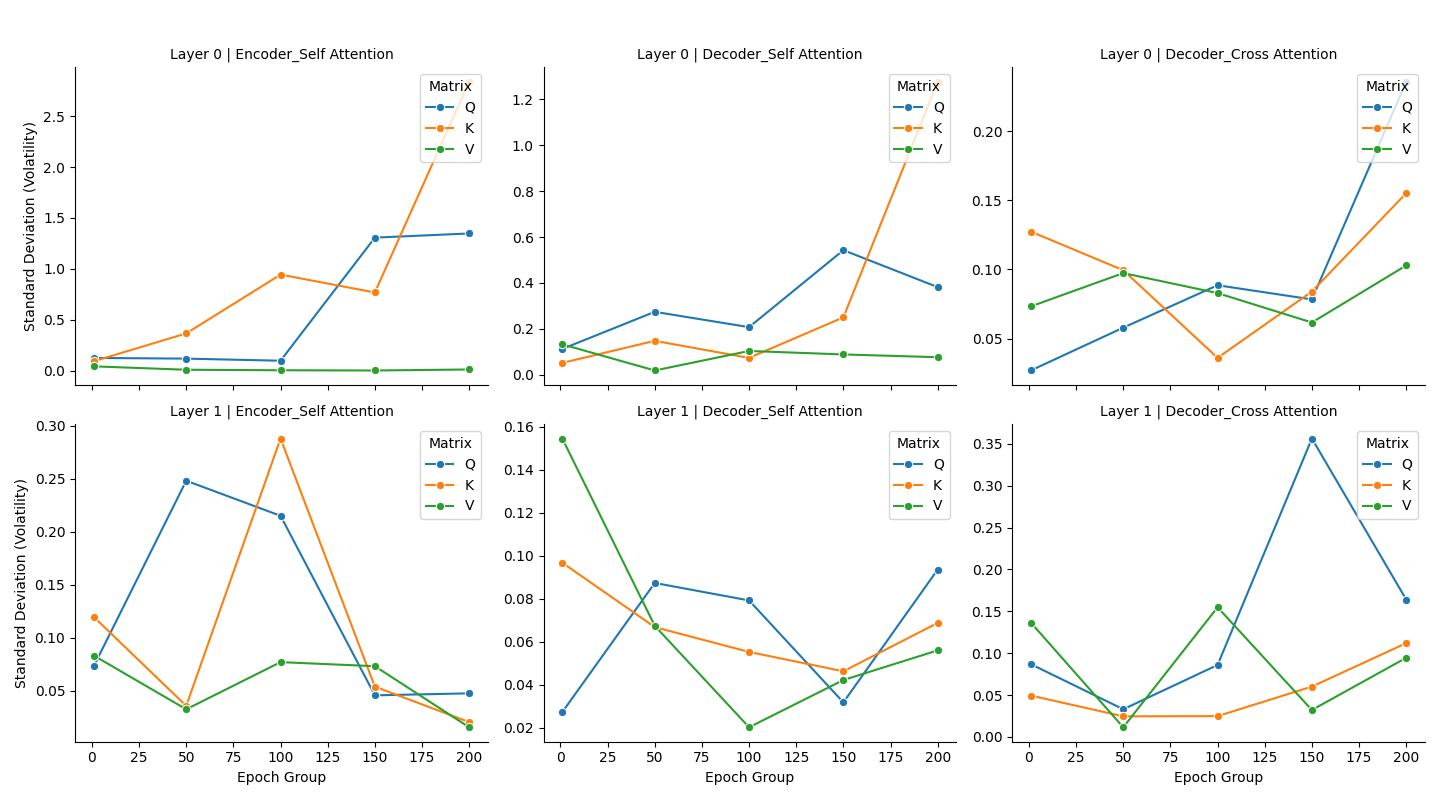}
		\caption{}
	\end{subfigure}
	\caption{$x_{min}$ parameter for T2\_En-Ch 9.62M}
	\label{fig: xmin for T2_En_Ch 9.62M}
\end{figure}

By examining the fitted parameters $\alpha$ and $x_{\min}$ across the two experimental configurations, we identify several salient patterns:

\begin{itemize}
	\item \textbf{Shallow layers are more stable than deep layers, and the cross-attention mechanism exhibits the most significant fluctuations:}
Across both the T1 and T2 models, the $\alpha$ trajectories of shallower layers (e.g., Layer-0) are notably smoother and more stable than those of deeper layers (e.g., Layer-1 and Layer-2) (see Figures~\ref{fig: alpha for T1_En_Ch 9.60M} and~\ref{fig: alpha for T2_En_Ch 9.62M}).In contrast, the $\alpha$ values of the decoder’s cross-attention layers (\texttt{Decoder-cross}) display the strongest randomness and volatility throughout training.This behavior likely reflects the inherently dynamic alignment process between the source and target language representations performed by the cross-attention mechanism, leading to slower convergence of the spectral properties of its weight matrices.
	
	\item \textbf{The encoder's self-attention layers (\texttt{Encoder-self}) at Layer-0 exhibit a highly consistent evolutionary pattern:}
The $\alpha$ values fluctuate substantially during the early training phase but gradually converge thereafter.  
After approximately 50 epochs, the $\alpha$ values stabilize within the range of 2-3.  
A smaller $\alpha$ corresponds to a ``heavier'' spectral tail, confirming that the empirical spectral density (ESD) of the matrix exhibits a pronounced heavy-tailed behavior as training progresses.  	
	\item \textbf{The fitted $x_{min}$ for most Q, K, V matrices shows an overall upward trend, with a few exceptions (notably, \texttt{en.0.s.a.v}):}
An increasing $x_{\min}$ suggests that, as training advances, the power-law behavior emerges among larger eigenvalues in the spectral distribution.  
However, the \texttt{en.0.s.a.v} matrix displays a distinct downward and convergent trend.  
This observation highlights the special role of this shallow V matrix, which may be responsible for encoding more fundamental or global features and could be more directly associated with the model’s overall convergence dynamics.  
\end{itemize}
\section{Technical Proofs for recognition criteria for Heavy-Tail based on PL fitting}
\subsection{Proof of Lemma \ref{lem:covergence of d}}\label{pro:lemma}
We analyze the convergence rate of the KS statistic  \(d\)  under the assumption that the power-law exponent \(\alpha\) is known. According to the Dvoretzky-Kiefer-Wolfowitz (DKW) inequality, a probabilistic bound exists for the maximum deviation between the ECDF and the true CDF for any finite sample. This inequality provides a rigorous characterization of the convergence rate of the KS statistic.
	\begin{equation}
		P(\sup_{x \in \mathbb{R}} |\hat{F}_n(x) - F(x)| > \epsilon) \leq 2e^{-2n\epsilon^2}.
	\end{equation}
	Consequently, we obtain
	\[
	P(d > \epsilon) \leq 2e^{-2n_{\text{tail}}\epsilon^2}.
	\]
	Let \(\epsilon = \frac{C}{\sqrt{n_{\text{tail}}}}\), where C is a constant. Substituting this into the inequality yields
	\[
	P\left(d > \frac{C}{\sqrt{n_{\text{tail}}}}\right) \leq 2e^{-2n_{\text{tail}}\left(\frac{C}{\sqrt{n_{\text{tail}}}}\right)^2} = 2e^{-2C^2}.
	\]
Therefore, the convergence rate of the KS distance \(d\) can be expressed as
	\[
	d = \sup_{x \geq x_{\min}} |\hat{F}_{n_{\text{tail}}}(x) - F(x; \alpha, x_{\min})| = O_p\left(\frac{1}{\sqrt{n_{\text{tail}}}}\right).
	\]
\subsection{Proof of Proposition \ref{prop:5.2}}\label{pro:prop}
Consider the KS distance
	\begin{align}
		\hat{d} &= \sup_{x \geq x_{\min}} |\hat{F}_{n_{\text{tail}}}(x) - F(x; \hat{\alpha}, x_{\min})|\nonumber\\
		&\le \sup_{x \geq x_{\min}} |\hat{F}_{n_{\text{tail}}}(x) - F(x; \alpha, x_{\min})| + \sup_{x\geq x_{min}} |F(x;\alpha,x_{min}) - F(x;\hat{\alpha},x_{min})|\nonumber\\
		&	= d + \sup_{x\geq x_{min}} |F(x;\alpha,x_{min}) - F(x;\hat{\alpha},x_{min})|.
	\end{align}
	We next show that $\sup_{x\geq x_{min}} |F(x;\alpha,x_{min}) - F(x;\hat{\alpha},x_{min})| = O\left(\frac{1}{\sqrt{n_{tail}}}\right)$.
	
	Log-likelihood function is
	\[
	l(\alpha) = \log L(\alpha) = \sum_{i=1}^{n_{\text{tail}}} \log p(x_i; \alpha),
	\]
	the MLE estimate \(\hat{\alpha}\) is the value that maximizes \(l(\alpha)\). Thus, it satisfies the first-order condition, where the score function is zero,
	\[
	S(\hat{\alpha}) = \left.\frac{\partial l(\alpha)}{\partial \alpha}\right|_{\alpha=\hat{\alpha}} = \sum_{i=1}^{n_{\text{tail}}} \left.\frac{\partial \log p(x_i; \alpha)}{\partial \alpha}\right|_{\alpha=\hat{\alpha}} = 0.
	\]
	Performing a first-order Taylor expansion of \(S(\hat{\alpha})\) around the true parameter \(\alpha_0\)
	\[
	0 = S(\hat{\alpha}) \approx S(\alpha_0) + (\hat{\alpha} - \alpha_0)S'(\alpha_0),
	\]
	
	\[
	\sqrt{n_{\text{tail}}}(\hat{\alpha} - \alpha_0) \approx \frac{\frac{1}{\sqrt{n_{\text{tail}}}}S(\alpha_0)}{-\frac{1}{n_{\text{tail}}}S'(\alpha_0)}.
	\]	
	By the Central Limit Theorem, the normalized form of the sum of i.i.d. random variables, \(\frac{1}{\sqrt{n_{\text{tail}}}}S(\alpha_0)\), converges in distribution to a normal distribution with a mean of 0 and a variance equal to the Fisher Information \(I(\alpha_0)\), where \(I(\alpha_0) = -E\left[\frac{d^2l(\alpha_0)}{d\alpha^2}\right]\).
	According to the Law of Large Numbers, the average of i.i.d. random variables, \(-\frac{1}{n_{\text{tail}}}S'(\alpha_0)\), converges in probability to its expected value, which is also the Fisher Information \(I(\alpha_0)\).
	
	Combining these results, we obtain that
	\[
	\sqrt{n_{\text{tail}}}(\hat{\alpha} - \alpha_0) \xrightarrow{d} \mathcal{N}\left(0, \frac{I(\alpha_0)}{I(\alpha_0)^2}\right) = \mathcal{N}\left(0, \frac{1}{I(\alpha_0)}\right).
	\]
	This is the Asymptotic Normality of MLE. This conclusion indicates that \(\sqrt{n_{\text{tail}}}(\hat{\alpha} - \alpha_0)\) converges to a random variable with finite variance. Therefore, \(\hat{\alpha} - \alpha_0\) must scale as \(1/\sqrt{n_{\text{tail}}}\) to counteract the magnifying effect of \(\sqrt{n_{\text{tail}}}\), i.e., $\hat{\alpha} - \alpha_0 = O_p\left(\frac{1}{\sqrt{n_{tail}}}\right)$.
	Hence,
	\begin{align}
		F(x;\alpha,x_{min}) - F(x;\hat{\alpha},x_{min}) =& \left(\frac{x}{x_{min}}\right)^{-\hat{\alpha}+1} - \left(\frac{x}{x_{min}}\right)^{-\alpha+1}\nonumber\\
		=& \left(\frac{x}{x_{min}}\right)^{-\alpha+1}\left(\left(\frac{x}{x_{min}}\right)^{-(\hat{\alpha}-\alpha)}-1\right),
	\end{align}
	Since $\sup_{x,\alpha} \left(\frac{x}{x_{min}}\right)^{-\alpha+1}$ is a constant and $x\ge x_{min}$, we have
	\begin{align}
		F(x;\alpha,x_{min}) - F(x;\hat{\alpha},x_{min}) = O_p\left(\frac{1}{\sqrt{n_{tail}}}\right).
	\end{align}
	
	Hence, the convergence rate of the distance \(\hat{d}\) is
	\[
	\hat{d} = O_p\left(\frac{1}{\sqrt{n_{\text{tail}}}}\right).
	\]
\section{Determination of Threshold C via Monte Carlo Simulation}\label{sec:Monte Carlo simulation}
To establish a robust and principled value for C, we employ a Monte Carlo simulation method for calibration. This approach empirically constructs the distribution of a statistic under the null hypothesis through a large number of repeated simulations, thereby enabling the selection of an appropriate threshold. The specific procedure is as follows.

\textbf{Step 1: Set Simulation Parameters}
\begin{itemize}
	\item \textbf{Theoretical Power-Law Exponent (\(\alpha_0\)):} To ensure the robustness of the threshold C across different tail shapes, we consider several typical values for the PL exponent, e.g., \(\alpha_0 = \{1.5, 2.0, 2.5, 3.0\}\).
	\item \textbf{Fixed Minimum Value (\(x_{\min}\)):} We adopt \(x_{\min} = 0\) as determined in our primary analysis.
	\item \textbf{Tail Data Count (\(n_{\text{tail}}\)):} To examine the sensitivity of C to sample size, several characteristic tail sizes are tested, e.g., 
\(n_{\text{tail}} in \{100, 200, 300\}\).
	\item \textbf{Number of Monte Carlo Iterations (k):} We set \(k = 10,000\) repetitions.
\end{itemize}

\textbf{Step 2: Generate Synthetic Data}

Synthetic datasets are generated under the null hypothesis, i.e., data perfectly follow a Power-Law distribution with the specified parameters (\(\alpha_0, x_{\min}\)). We employ inverse transform sampling:
\[
x = x_{\min} \cdot (1 - u)^{-\frac{1}{\alpha_0 - 1}},
\]
where \(u\) is drawn uniformly from [0, 1]. For each combination of (\(\alpha_0, n_{\text{tail}}\)), this process is repeated k times to obtain k independent synthetic datasets.

\textbf{Step 3: Calculate the Normalized Statistic S}

For each of the k synthetic datasets, we perform the following calculations:

\begin{enumerate}
	\item Estimate the PL exponent \(\hat{\alpha}\) via Maximum Likelihood Estimation (MLE).
	\item Compute the KS distance, \(\hat{d}\), between the \(\hat{F}_{n_{\text{tail}}}(x)\) of the dataset and \(F(x; \hat{\alpha}, x_{\min})\) using the estimated parameter \(\hat{\alpha}\).
	\item Calculate the normalized KS statistic S:
	\[
	S = \hat{d} \cdot \sqrt{n_{\text{tail}}}, \quad \text{where} \ \hat{d} = \sup_{x \geq x_{\min}} |\hat{F}_{n_{\text{tail}}}(x) - F(x; \hat{\alpha}, x_{\min})|.
	\]
\end{enumerate}

\textbf{Step 4: Select Threshold C}

After k iterations, the resulting set of S values constitutes the empirical distribution of the normalized KS statistic under the null hypothesis.
\begin{itemize}
	\item \textbf{Distribution Visualization:} Plot a histogram of the k values of S. This histogram provides a visual representation of the empirical probability distribution of the normalized KS statistic S under the null hypothesis.
	\item \textbf{Selection of the Threshold C:} Determine a ``safe'' upper bound from the histogram. This value, C, should exceed nearly all simulated S values, ensuring a conservative and robust threshold.
\end{itemize}

Following the steps outlined above, the results of our simulation are presented in Figure~\ref{fig:monte_carlo_calibration}.
The simulation results indicate that, across 10,000 Monte Carlo iterations for various values of  $\alpha$ and tail data sizes ($n_{\text{tail}}$), the upper bound of the normalized KS statistic remains highly stable and does not exceed 2. Consequently, we select \textbf{C=2} as the critical threshold.
\begin{figure}[h]
	\centering \includegraphics[width=0.45\textwidth]{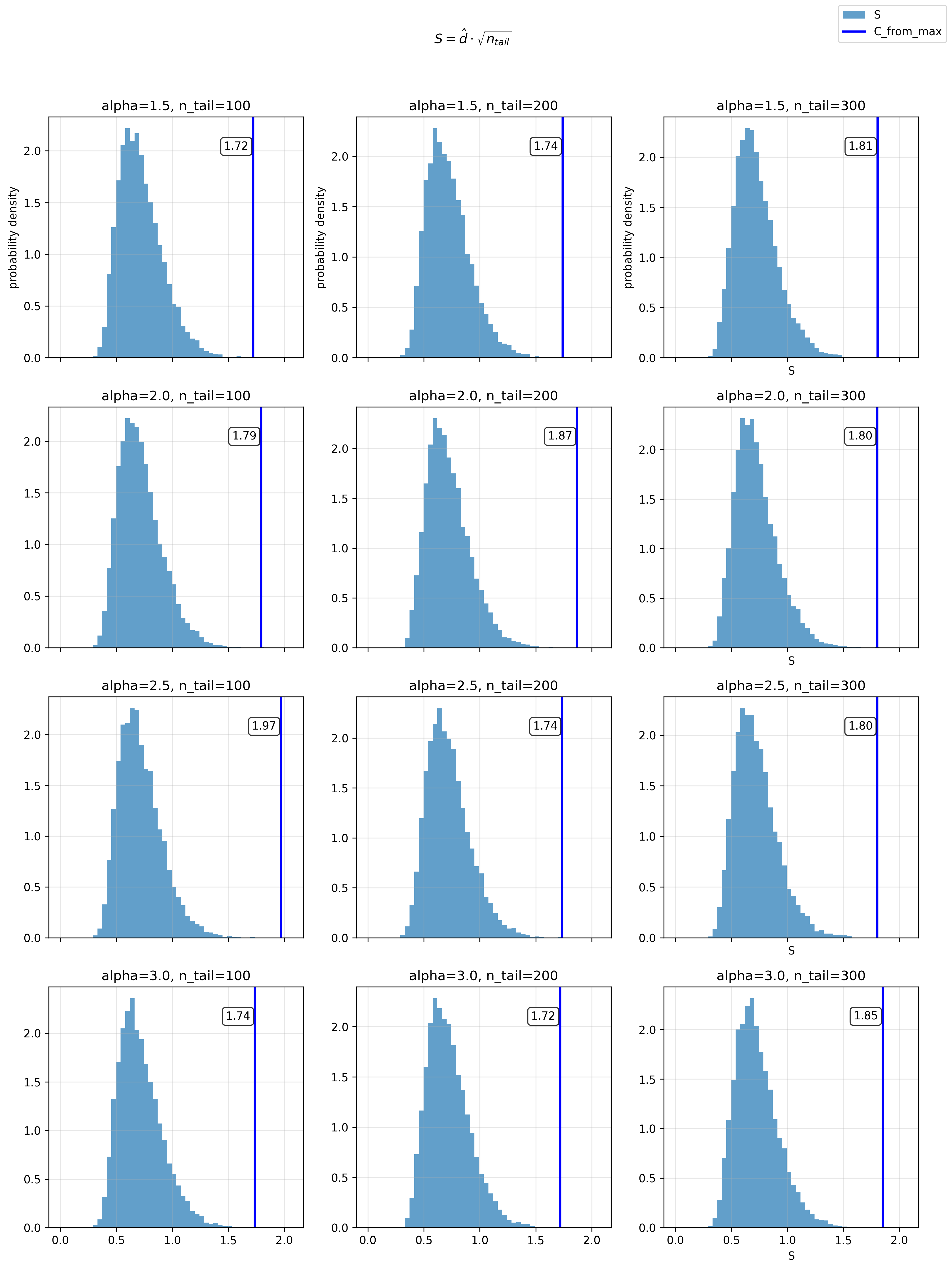}
	\caption{Monte Carlo simulation for calibrating the threshold constant C. Each subplot displays the distribution of the statistic $S = \hat{d} \cdot \sqrt{n_{\text{tail}}}$ over 10,000 runs, for different true values of $\alpha$ and tail sizes ($n_{\text{tail}}$). The vertical blue line indicates the proposed threshold.}
	\label{fig:monte_carlo_calibration}
\end{figure}
\section{Evolutionary Trajectory of the Early-Stopping Metric for \textbf{T2\_En-Ch 9.62M} and \textbf{T3\_En-Ch 19.20M}}\label{sec:Evolutionary_Trajectory}
\begin{figure}[H]
	\centering
	\begin{subfigure}[b]{0.48\linewidth}
		\centering
		\includegraphics[width=0.8\linewidth]{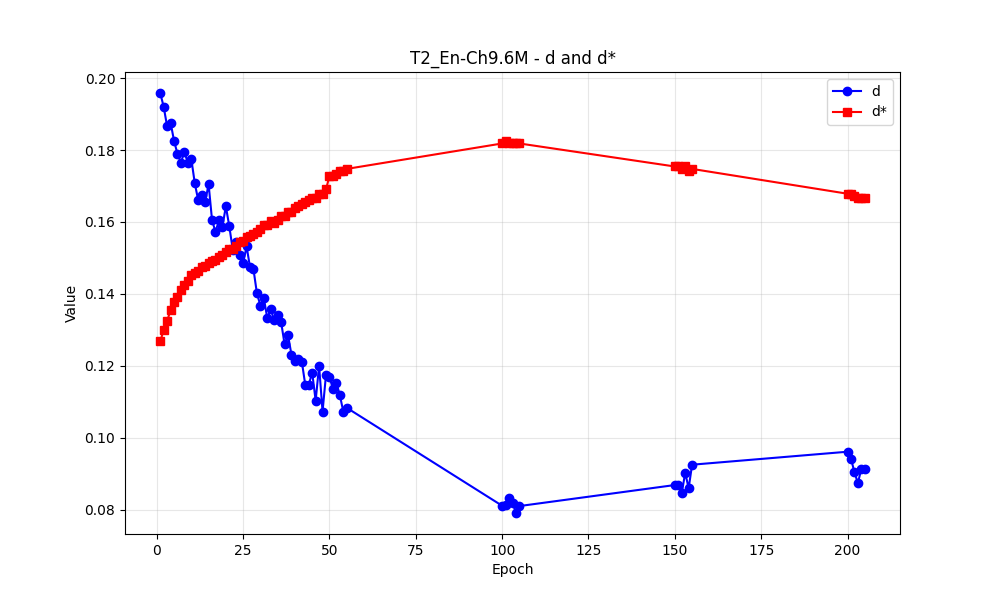}
		\caption{}
	\end{subfigure}
	\begin{subfigure}[b]{0.48\linewidth}
		\centering
		\includegraphics[width=0.8\linewidth]{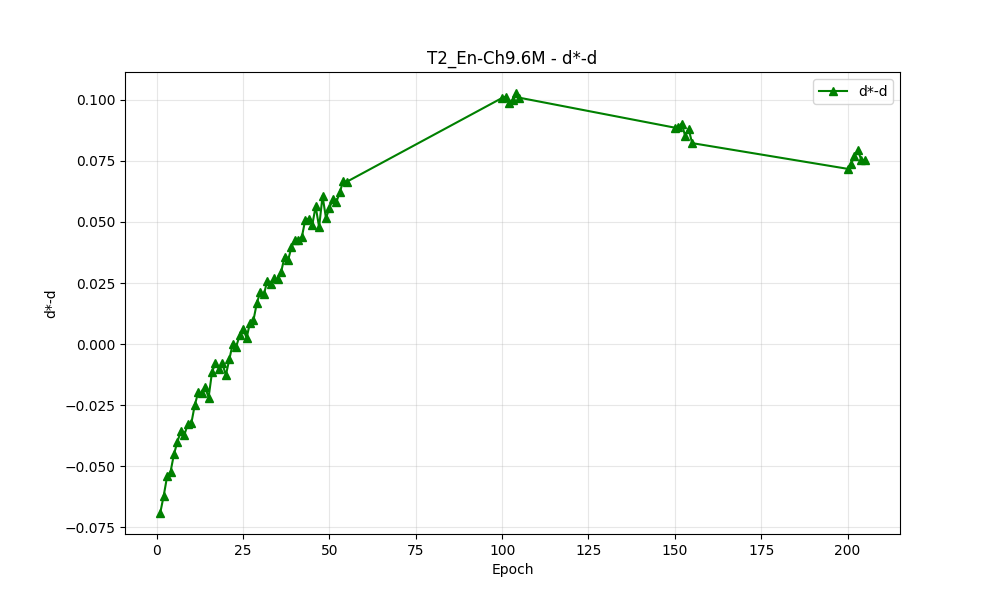}
		\caption{}
	\end{subfigure}
	
	\begin{subfigure}[b]{0.48\linewidth}
		\centering
		\includegraphics[width=0.8\linewidth]{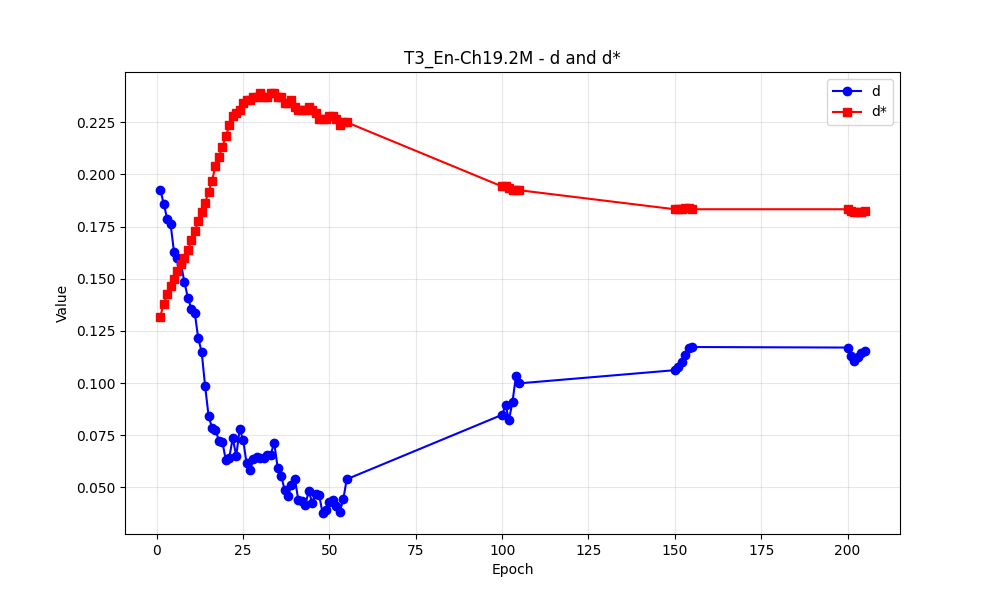}
		\caption{}
	\end{subfigure}
	\begin{subfigure}[b]{0.48\linewidth}
		\centering
		\includegraphics[width=0.8\linewidth]{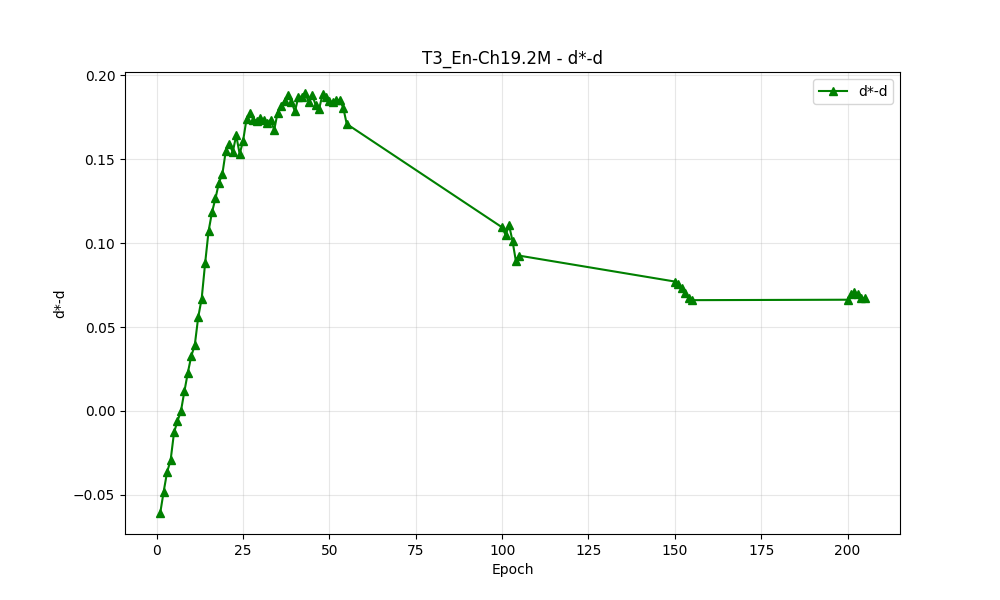}
		\caption{}
	\end{subfigure}
	\caption{Evolutionary Trajectory of the Early-Stopping Metric. (a) and (b) describe the variation of d and d* with training epochs for \textbf{T2\_En-Ch 9.62M} and \textbf{T3\_En-Ch 19.20M}, respectively; while (b) and (d), illustrate the changes in the difference between d* and d across training epochs for these two models, respectively.}
\end{figure}

\end{document}